\documentclass[
10pt, %
a4paper, %
oneside, %
headinclude,footinclude, %
BCOR5mm, %
]{scrartcl}

\usepackage[
nochapters, %
beramono, %
eulermath,%
pdfspacing, %
dottedtoc, %
]{classicthesis} %

\usepackage{arsclassica} %

\usepackage[T1]{fontenc} %

\usepackage[utf8]{inputenc} %

\usepackage{graphicx} %

\usepackage{enumitem} %

\usepackage{lipsum} %

\usepackage{subfig} %

\usepackage{amsmath,amssymb,amsthm} %

\usepackage{varioref} %

\theoremstyle{definition} %
\newtheorem{definition}{Definition}

\theoremstyle{plain} %
\newtheorem{theorem}{Theorem}

\theoremstyle{remark} %

\usepackage{hyperref} %
\hypersetup{
colorlinks=true, breaklinks=true, bookmarks=true,bookmarksnumbered,
urlcolor=webbrown, linkcolor=RoyalBlue, citecolor=webgreen, %
pdftitle={}, %
pdfauthor={\textcopyright}, %
pdfsubject={}, %
pdfkeywords={}, %
pdfcreator={pdfLaTeX}, %
pdfproducer={LaTeX with hyperref and ClassicThesis} %
} %
\usepackage{mathtools}

\usepackage[acronym]{glossaries}
\glsdisablehyper

\newcommand{\twobytwo}[4]{\begin{pmatrix}#1 & #2 \\#3 & #4\end{pmatrix}}
\newcommand{\twobyone}[2]{\begin{pmatrix}#1 \\ #2 \end{pmatrix}}
\newcommand{\threebythree}[9]{\begin{pmatrix}#1 & #2 & #3 \\#4 & #5 & #6\\#7 & #8 & #9 \end{pmatrix}}

\newcommand{\norm}[1]{\left\lVert#1\right\rVert}
\newcommand{\fun}[3]{\ensuremath{#1\colon #2\to #3}}
\newcommand{\bundle}[3]{\ensuremath{#1 \xrightarrow{#2} #3}}

\DeclarePairedDelimiter\inner{\langle}{\rangle}
\DeclarePairedDelimiter\abs{\lvert}{\rvert}

\newcommand{\R}{\mathbb{R}}
\newcommand{\C}{\mathbb{C}}
\newcommand{\Z}{\mathbb{Z}}
\newcommand{\N}{\mathbb{N}}
\newcommand{\Q}{\mathbb{Q}}
\newcommand{\I}{\mathcal{I}}

\newcommand{\GL}{\ensuremath{\mathbf{GL}}}
\newcommand{\SO}{\ensuremath{\mathbf{SO}}}
\newcommand{\SE}{\ensuremath{\mathbf{SE}}}

\newcommand{\SU}{\ensuremath{\mathbf{SU}}}

\newcommand{\SL}{\ensuremath{\mathbf{SL}}}

\newcommand{\tr}{\text{tr}}

\renewcommand{\l}{\ell}

\newcommand{\Hom}{\text{Hom}}

\newacronym{FFT}{FFT}{Fast Fourier Transform}
\newacronym{CNN}{CNN}{convolutional neural network}
\newacronym{G-CNN}{G-CNN}{group equivariant convolutional neural network}
\newcommand{\cnn}{\gls{CNN}}
\newcommand{\cnns}{\glspl{CNN}}
\newcommand{\gcnn}{\gls{G-CNN}}
\newcommand{\gcnns}{\glspl{G-CNN}}
\newcommand{\fft}{\gls{FFT}}

\usepackage{microtype}
\usepackage{siunitx}
\usepackage[uniquelist=false,mincitenames=1,maxcitenames=2,maxbibnames=99]{biblatex}

\usepackage[capitalize]{cleveref}
\usepackage{tikz-cd}

\theoremstyle{theorem}
\newtheorem{proposition}{Proposition}
\newtheorem{lemma}{Lemma}
\theoremstyle{remark}
\newtheorem*{remark}{Remark}
\newtheorem*{example}{Example}
\newtheorem*{examples}{Examples}

\newcommand{\h}{\mathfrak{h}}
\newcommand{\da}{\downarrow\!\!}
\newcommand{\ua}{\uparrow\!\!}
\DeclareMathOperator{\vectorize}{vec}

\title{\Large\normalfont\spacedallcaps{Theoretical aspects of group equivariant neural networks}} %

\author{\spacedlowsmallcaps{Carlos Esteves}} %

\date{\normalsize
  Department of Computer and Information Science \\
  University of Pennsylvania \\
  Philadelphia, PA} %

\begin{document}

\renewcommand{\sectionmark}[1]{\markright{\thesection\enspace\spacedlowsmallcaps{#1}}}
\lehead{\mbox{\llap{\small\thepage\kern1em\color{halfgray} \vline}\color{halfgray}\hspace{0.5em}\rightmark\hfil}} %

\pagestyle{scrheadings} %

\maketitle %

\setcounter{tocdepth}{2} %

\section*{Abstract} %
Group equivariant neural networks have been explored in the past few years and are interesting from theoretical and practical standpoints.
They leverage concepts from group representation theory, non-commutative harmonic analysis and differential geometry that do not often appear in machine learning.
In practice, they have been shown to reduce sample and model complexity,
notably in challenging tasks where input transformations such as arbitrary rotations are present.
We begin this work with an exposition of group representation theory and the machinery necessary to define and evaluate integrals and convolutions on groups.
Then, we show applications to recent \SO(3) and \SE(3) equivariant networks,
namely the Spherical CNNs, Clebsch-Gordan Networks, and 3D Steerable CNNs.
We proceed to discuss two recent theoretical results.
The first, by Kondor and Trivedi (ICML'18), shows that a neural network is group equivariant if and only if it has a convolutional structure.
The second, by Cohen et al. (NeurIPS'19), generalizes the first to a larger class of networks,
with feature maps as fields on homogeneous spaces.
\newpage
\tableofcontents %
\newpage

\section{Introduction}
Recall the familiar \emph{convolution} of functions $f$ and $k$ on the real line
\begin{align*}
  (f * k)(x) = \int\limits_{t\in \R} f(t)k(x-t)\, dt.
\end{align*}
We define the shift operator $(\lambda_yf)(x) = f(x-y)$.
One important property of convolution is the \emph{shift-equivariance}:
$(\lambda_y f) * k = \lambda_y (f * k)$.
Intuitively, if the filter $k$ is designed to respond to some pattern in $f$,
this property tells us that the response will be the same (just shifted) no matter where the
pattern appears.
This is fundamental to the success of \cnns\ introduced by \textcite{fukushima1980neocognitron}.
The filter $k$ is \emph{learned} and compactly supported,
and convolution allows weight-sharing, in contrast with fully connected networks.
The combination of \cnns\ and the backpropagation algorithm (\textcite{lecun1989backpropagation}) enabled the recent deep learning revolution.
The vast majority of \cnns\ employ convolution on Euclidean spaces.
For example, $\R$ for audio, $\R^2$ for images and $\R^3$ for volumetric occupancy grids.

The main objective of this work is to present a theory of \cnns\ on more general spaces;
namely, groups and homogeneous spaces.
We still wish the equivariance property to hold,
but we will now call it \emph{group-equivariance}
and define it with respect to the action of a group.
Let $\lambda_g,\, \lambda'_g$ indicate the group action for some $g \in G$.
We say that a linear map \fun{\Phi}{X}{Y} is equivariant to actions of $G$ when
\begin{align}
  \Phi(\lambda_g f) = \lambda'_g (\Phi (f)),
  \label{eq:equivariance}
\end{align}
which is equivalently represented by the commutative diagram
\[
  \begin{tikzcd}
    X \arrow{r}{\Phi} \arrow[swap]{d}{\lambda_g} & Y \arrow{d}{\lambda'_g} \\
    X \arrow{r}{\Phi} & Y.
  \end{tikzcd}
\]
We are interested in designing and parameterizing $\Phi$ such that it is equivariant and its parameters can be optimized.
Note that $\lambda$ and $\lambda'$ are not necessarily the same since
$\Phi$ may map between different spaces.
When $\lambda'$ is the identity, we say that $\Phi$ is \emph{invariant} to $G$.
Some authors reserve the term \emph{equivariant} for when $\lambda=\lambda'$ and call it
\emph{covariant} otherwise, but we will not make this distinction.
\cnns\ satisfying this property are called \gcnns,
introduced by \textcite{cohen2016group}.

A typical example are \gcnns\ for inputs on the surface of the sphere $S^2$,
where the group of rotations $\SO(3)$%
\footnote{$\SO(3)$ is the group of special orthogonal $3\times 3$ matrices,
 which is identified with 3D rotations.}
act.
One application is the semantic segmentation of a panoramic images~\cite{sphhg,CohenWKW19}
where the equivariance property enforces that outputs strictly follow
input camera rotations.

Our focus is on the theoretical background that enables \gcnns,
which is not usually covered in recent papers due to space constraints.
To this end, we will discuss group representation theory (\cref{sec:group}),
integration and harmonic analysis on non-Euclidean spaces (\cref{sec:haar,sec:harm}).
\Cref{sec:appl} shows how this theory is applied to \gcnns.
\Cref{sec:kt,sec:cohen} cover recent results by \textcite{kondor18general,cohen2019general} generalizing the \gcnn\ theory and showing that equivariance implies a group convolutional structure.
Before discussing \textcite{cohen2019general},
we need to introduce some concepts from differential geometry related to fiber bundles,
which is done in \cref{sec:fiber-bundles}.

We will not cover recent results of \gcnns\
operating on 2D images, where equivariance to
translation, planar rotation, mirroring and sometimes scale is sought.
We recommend \textcite{weiler2019general} for a general theory
of this class of networks.

Most of the material in \cref{sec:group,sec:haar,sec:harm}
is presented in a more rigorous and complete way in \textcite{gallierncharm}.
We omit deep proofs related to the Haar measure and the Peter-Weyl theorem,
and often tailor the material to just the parts required to understand the current
\gcnns.
We will, nevertheless, derive the irreducible representations of
$\SL(2, \C)$, $\SU(2)$ and $\SO(3)$ and show how special functions,
including the spherical harmonics, arise in the process.
Furthermore, we define and prove the formulas for  $\SO(3)$ and spherical convolutions
and cross-correlations that are used in recent works.
While \textcite{gallierncharm} is the main reference utilized,
we sometimes follow \textcite{nakahara2003geometry,serre1977linear,dieudonn1980special,folland2016course,hall2015lie,vilenkin1978special,rudin2006real} when more appropriate.

No new theoretical or practical results are introduced in this work.
We present an introduction to the theoretical background
and recent developments under a consistent notation,
and sometimes from a different point of view,
which we hope will be useful for some readers.

\section{Group representation theory }
\label{sec:group}
Group representation theory is the study of groups by the way they act on vector spaces,
which is done by \emph{representing} elements of the group as linear maps between vector spaces.
\subsection{Groups and homogeneous spaces}
We begin with basic definitions about groups.
\begin{definition}[group]
A \emph{group} $(G, \cdot)$ is a set $G$ equipped with an associative binary operation $\fun{\cdot}{G \times G}{G}$,
an identity element, and where every element has an inverse also in the set.
When $\cdot$ is commutative, we call the group \emph{abelian} or \emph{commutative}.
When the set is equipped with a topology where $\cdot$ and the inverse map are continuous,
we call it a \emph{topological group}.
When such topology is compact, we call the group a \emph{compact group}.
When $G$ is a \emph{smooth manifold} and $\cdot$ and the inverse map are smooth, it is a \emph{Lie group}.
A \emph{subgroup} $(H, \cdot)$ of a group $(G, \cdot)$ is a group such that $H \subseteq G$.
\end{definition}
\begin{examples}
  \begin{itemize}
  \item[]
  \item The integers under addition $(\Z, +)$ form an abelian,
    non-compact group.
  \item The group of all permutations of a set of $n$ symbols,
called the symmetric group $S_n$ is a finite, non-commutative group of $n!$ elements.
  \item The group of rotations in $3$D, \SO(3), is a compact, non-commutative Lie group.
  \end{itemize}
\end{examples}
\noindent For a negative example,
consider the sphere $S^2$ and its north pole $\nu=(0, 0, 1)$.
We can identify any point on the sphere by angles $(\alpha, \beta)$,
which represent a rotation of the north pole $R_y(\beta)$ (around $y$)
followed by $R_z(\alpha)$ (around $z$); we write
${x = R_z(\alpha)R_y(\beta)\nu}$.
Now define the operation $\cdot$ such that
$x_1 \cdot x_2 =  R_z(\alpha_1)R_y(\beta_1)R_z(\alpha_2)R_y(\beta_2)\nu$.
Any rotation in $\SO(3)$ can be represented as $R_z(\alpha_1)R_y(\beta_1)R_z(\alpha_2)R_y(\beta_2)$,
and not only the ones of the form $R_z(\alpha_1)R_y(\beta_1)$;
therefore the operation $\cdot$ as defined is not closed in $S^2$,
and $(S^2, \cdot)$ is not a group.

While $S^2$ is not a group, we will show that it is a homogeneous space
of $\SO(3)$.
Intuitively, homogeneous spaces are spaces where the group acts ``nicely''.
For this reason, they are useful as the feature domain in \gcnns.
Homogeneous spaces are closely related to coset spaces;
we now define both structures and show how they relate.
\begin{definition}[homogeneous space]
  The action of a group $G$ is \emph{transitive} on a space $X$
  if for any pair of elements $x$ and $y$ in $X$,
  there exists an element $g$ in $G$ such that $y = gx$.
  A \emph{homogeneous space} $X$ of a group $G$ is a space where the group acts transitively.
\end{definition}

\begin{definition}[coset space]
  Given a subgroup $H$ and an element $g$ of a group $G$,
  we define the \emph{left coset} $gH$ as ${gH = \{gh \mid h \in H\}}$.
  The set of left cosets partition $G$ and is called the \emph{left coset space} $G/H$.
  We define the right cosets $Hg$ and their coset space $H\backslash G$ analogously.
\end{definition}
\noindent Let $o \in X$ be an arbitrarily chosen origin of $X$ and $H_o$ its stabilizer.
Then, there is a bijection%
\footnote{The bijection will be a homeomorphism is all cases considered in this work,
  but not in general.}
between $X$ and $G/H_o$.

We will often refer to elements of a homogeneous space $X \cong G/H_o$ by the coset $gH_o$,
and the map $g \mapsto gH_o$ is a projection from the group $G$ to the homogeneous space $X$.
Since we are interested in maps that are equivariant to actions of some group $G$,
we will frequently consider maps between homogeneous spaces of $G$.
\begin{example}
  Let us return to the sphere $S^2$ and its north pole $\nu=(0, 0, 1)$.
  The sphere is a homogeneous space since $\SO(3)$ acts transitively on it.
  The set of rotations that do not move $\nu$ is the stabilizer
  $H_\nu = \{R_z(\delta) \mid \delta \in [0, 2\pi) \}$.
  Any rotation in $R\in \SO(3)$ can be written as $R = R_z(\alpha)R_y(\beta)R_z(\gamma)$,
  and generate left cosets of the form
  ${RH_\nu = \{R_z(\alpha)R_y(\beta)R_z(\gamma+\delta) \mid \delta \in [0, 2\pi) \}}$.
  The pair $(\alpha,\, \beta)$ uniquely identify each coset, which gives an isomorphism
  between points on the sphere and the set of all cosets $\SO(3)/H_\nu$.
  Since $H_\nu$ is isomorphic to group of planar rotations $\SO(2)$,
  we write $S^2 \cong \SO(3) / \SO(2)$.
\end{example}

\subsection{Group representations}
\label{sec:representations}
Group representations have numerous applications.
Most important to our purposes are
(i) they represent actions on vector spaces
(for example, $\lambda_g$ in \cref{eq:equivariance} could be a linear representation),
and (ii) they form bases for spaces of functions on groups, as will be detailed in \cref{sec:harm}.

\begin{definition}[representation]
  A \emph{group homomorphism} between groups $G$ and $H$ is a map \fun{f}{G}{H} such that $f(g_1g_2) = f(g_1)f(g_2)$.
  Let $G$ be a group and $V$ a vector space over some field.
  A \emph{linear representation} is a group homomorphism \fun{\rho}{G}{\GL(V)},
  where $\GL(V)$ is the general linear group.%
  \footnote{When $V$ is finite-dimensional and $n = \dim V$, $\GL(V)$ is identifiable with the group of $n \times n$ invertible matrices.}
  If $V$ is an inner product space and $\rho$ is continuous and preserves the inner product,
  it is called a \emph{unitary representation}.
  The \emph{character} of a representation $\rho$ is the map \fun{\chi_\rho}{G}{\C}
  such that $\chi_\rho(g) = \tr(\rho(g))$.
\end{definition}
\begin{example}
  Consider the multiplicative group $G$ of complex numbers of the form $g_\theta = e^{i\theta}$.
  The map
  \[\rho(e^{i\theta}) = \twobytwo{\cos \theta}{ \sin\theta}{-\sin\theta}{\cos\theta}\]
  is a representation of $G$ on $\R^2$.
  We can check that $g_\theta g_\phi = g_{\theta+\phi}$ and
  $\rho(e^{i(\theta + \phi)}) = \rho(e^{i\theta})\rho(e^{i\phi})$.
\end{example}
\begin{example}
  Let $L^2(G)$ be the Hilbert space of square integrable functions on $G$, and
  let \fun{\rho}{G}{\GL(L^2(G))} act on \fun{f}{G}{\C} as
  $(\rho(g)(f))(x) = f(g^{-1}x)$.
  $\rho$ defined this way is a representation of $G$; specifically,
  it is a \emph{left regular representation} of $G$.
\end{example}
\begin{definition}[irreducible representation]
  Let \fun{\rho}{G}{\GL(V)} be a representation of $G$ on a vector space $V$,
  and $W$ be a vector subspace of $V$.
  When $W$ is invariant under the action of $G$, i.e.,
  for all $g \in G$ and $w\in W$ we have $\rho(g)(w) \in W$,
  the restriction of $\rho$ to $W$ is a representation of $G$ on $W$, called a \emph{subrepresentation}.
  When the only subrepresentations of $\rho$ are $V$ and the zero vector space,
  we call $\rho$ an \emph{irreducible representation} or \emph{irrep}.
\end{definition}
\begin{example}
  Consider the group $\SO(3)$ and the vector space $V$ of $3\times 3$ real matrices ($V \cong \R^9$).
  We define a representation \fun{\rho}{\SO(3)}{\GL(V)} such that ${\rho(g)(A) = g^\top A g}$.
  Now consider the subspace $W$ of $V$ comprised of antisymmetric matrices ($B=-B^\top$).
  It turns out $W$ is invariant to $\rho$,
  \begin{align}
    (\rho(g)(B))^\top = (g^\top B g)^\top = -g^\top B g = -\rho(g)(B)
  \end{align}
  so $\rho(g)(B) \in W$ for all $g \in \SO(3)$ and $B \in W$.
  Therefore $\rho$ is a reducible representation.
  It is, however, irreducible as a representation on $W$.
\end{example}
\begin{remark}
Every representation of a finite group is a direct sum of irreps (Maschke's theorem).
\end{remark}
\begin{remark}
Every finite-dimensional unitary representation of a compact group is a direct sum of unitary irreducible representations (unirreps).
\end{remark}

We often want to determine all irreducible representations of a group,
or decompose a representation in its irreducible parts.
The characters ${\fun{\chi_\rho}{G}{\C}}$ play an important role in this task.
First, we define the inner product of characters $\inner{\chi_i, \chi_j} = \int_G \chi_i(g)\overline{\chi_j(g)}\, dg$.%
\footnote{This involves integration on the group, which we will define in \cref{sec:haar}.}
The following properties hold:
\begin{itemize}
\item Isomorphic representations have the same character.
  The converse is only true for semisimple representations,
  which include unitary representations and all representations of finite or compact groups.
\item Distinct characters of irreducible representations of compact groups are orthogonal,
  $\inner{\chi_i, \chi_j} = 0$ when $i \neq j$.
\item A representation of a compact group is irreducible if and only if its character $\chi$ satisfies $\inner{\chi, \chi} = 1$.
\item The character of a direct sum of representations is the sum of the individual characters.
\end{itemize}

Now let \fun{\rho_1}{G}{\GL(V_1)} and \fun{\rho_2}{G}{\GL(V_2)} be finite-dimensional representations.
The map \fun{\rho_{12}}{G}{V_1 \otimes V_2}
obtained via tensor product $\rho_{12}(g) = \rho_1(g) \otimes \rho_2(g)$
is a representation of $V_1 \otimes V_2$.
This representation is not irreducible in general,
and the Clebsch-Gordan theory studies how it decomposes into irreps.
\begin{definition}[G-map]
  Given two representations \fun{\rho_1}{G}{\GL(V_1)} and \fun{\rho_2}{G}{\GL(V_2)},
  a \emph{G-map} is a linear map \fun{f}{V_1}{V_2} such that $${f(\rho_1(g)(v_1)) = \rho_2(g)(f(v_1))}$$
  for every $g \in G$ and $v_1\in V_1$.
  If $f$ is invertible, we say that $\rho_1$ and $\rho_2$ are \emph{equivalent},
  and we can define equivalence classes of representations.
  A G-map is sometimes called a \emph{G-linear}, \emph{G-equivariant}, or \emph{intertwining} map.
\end{definition}
\begin{remark}
  In the context of neural networks, we usually have alternating linear maps and nonlinearities.
  In equivariant neural networks, we want the linear maps to be G-maps.
  The representations will often be the natural action $(\rho(g)f)(x) = f(g^{-1}x)$.
\end{remark}

The following is an important result characterizing G-maps between irreps.
\begin{theorem}[Schur's Lemma]
\label{thm:schur}
Let \fun{\rho_1}{G}{\GL(V_1)} and \fun{\rho_2}{G}{\GL(V_2)} be irreducible representations of $G$,
and \fun{f}{V_1}{V_2} a \emph{G-map} between them.
Then $f$ is either zero or an isomorphism.
If $V_1=V_2$ and $\rho_1=\rho_2$ are complex representations, then $f$ is a multiple of the identity map, $f = \lambda \text{id}$.
\end{theorem}
Henceforth, we assume representations are complex (representation vector space is over a complex field)
except when stated otherwise.

It is possible to construct a representation of a group $G$
from a representation $\rho$ of a subgroup $H$ on a vector space $V$.
The basic idea is to consider $V$-valued functions on $G$
and let $\rho$ act on their values.
This is key to define representations on associated vector bundles,
which will be discussed in \cref{sec:fiber-bundles,sec:cohen}.
\begin{definition}[induced representation]
  \label{def:induced}
  Let $G$ be a locally compact group with a subgroup $H$,
  and \fun{\rho}{H}{\GL(V)} a representation of $H$.
  Consider the space $W$ of vector valued functions \fun{f}{G}{V}
  such that $f(gh) = \rho(h^{-1})(f(g))$,
  and the operator $\pi$ on $W$ such that $(\pi(u)(f))(g) = f(u^{-1}g)$.
  It is easy to see that $\pi$ is a linear map from $W$ to itself and
  a group homomorphism; hence, a representation of $G$ on $W$.
  We call $\pi$ the representation of $G$ \emph{induced} by $\rho$,
  sometimes denoted $\text{Ind}_H^G\rho$.
\end{definition}
\begin{remark}
We can construct a suitable function $f$
from any continuous {\fun{k}{G}{V}} with compact support
by making $f(g) = \int_H\rho(h)k(gh)\, dh$.
\end{remark}
\begin{example}
  Consider $G$, $H$ and $f$ in \cref{def:induced} and let $\rho$ be
  the trivial representation $\rho(h) = I$ for all $h$. Then, $f(gh)=f(g)$, which implies that
  $f$ is constant on cosets, and thus can be viewed as a function on G/H.
  The induced representation $\pi$ is then just the natural representation
  of $G$ on $L^2(G/H)$, given by $(\pi(u) f)(gH) = f(u^{-1}gH)$.
\end{example}

\begin{example}
  There is an alternative but equivalent geometric construction of the induced representation
  using vector bundles (see \cref{def:assoc}).
  We discuss a simple example; refer to \textcite{folland2016course} for the full details.
  Consider the group of rigid motions in the plane $\SE(2)$, its subgroup $\SO(2)$,
  and a representation of \SO(2) on $\R^2$, \fun{\rho}{\SO(2)}{\GL(\R^2)}.
  We refer to elements of \SE(2) as $(t,r)$ for a translation $t$ and rotation $r$.
  Now let $V$ be the space of functions \fun{f}{\R^2}{\R^2} (vector fields on the plane).
  The representation $\rho$ of $\SO(2)$ acts on the range of $f$,
  and induces a representation \fun{\pi}{\SE(2)}{\GL(V)}
  \[\pi((t, r))(f)(x) = \rho(r) f(r^{-1}(x-t)). \]
  Geometrically, this is an action of $\SE(2)$ on the vector field that
  first change the coordinates, then rotates the vectors.
\end{example}

This concludes our introduction to group representation theory.
For more details we recommend \textcite{gallierncharm,serre1977linear,hall2015lie}.
\section{Integration}
\label{sec:haar}
In order to compute Fourier transforms and convolutions on groups,
we need to integrate functions on groups.
The key ingredient is the Haar measure.
We begin with the familiar Riemann integral,
discuss its limitations and introduce Lebesgue integration as the remedy.
The Lebesgue integral allows integration over arbitrary sets
given an appropriate measure.
Finally, we define the Haar measure,
which is the appropriate measure used for integration on
locally compact groups.
\subsection{The Riemann integral}
Intuitively, the Riemann integral is the familiar ``area under the curve''
of a continuous function on an interval of the real line \fun{f}{[a,b]}{\R}.
The idea is to partition the integration interval and
define the integral as the sum of areas of the rectangles defined by one value
of $f$ on each subinterval and the subinterval width,
on the limit where such widths tend to zero.

\begin{definition}[Riemann integral]
  For an interval $[a,b] \subset \R$ and a subdivision $T=\{t_i\}$ with $0 \le i \le n$,
  $t_0=a$, $t_{n}=b$, and $t_k < t_{k+1}$ for all $k < n$,
  the \emph{Cauchy-Riemann sum} $s_T(f)$ of a continuous function \fun{f}{[a,b]}{\R} is
  \begin{align*}
    s_T(f) = \sum_{k=0}^{n-1}(t_{k+1}-t_{k})f(t_k).
  \end{align*}
  The \emph{diameter} of the subdivision $T$ is $\delta(T) = \max_k\{t_{k+1}-t_k\}$.
  Now consider any sequence of subdivisions $T_j$ such that $\lim_{j \to \infty}\delta(T_j) = 0$
  (consequently, $n \to \infty$).
  We define the \emph{Riemann integral} as $\int_a^b f(t)\,dt = \lim_{j \to \infty} s_{T_j}(f)$.
\end{definition}
It can be shown that $s_{T_j}$ always converge to the same limit
for any sequence of subdivisions $T_j$.
Importantly, the map $f \mapsto \int_a^bf(t) \, dt$ is linear.
The Riemann integral can be extended to functions on products of closed intervals on $R^n$
and to vector valued functions.
However, it cannot be defined on more general domains;
the Lebesgue integration was introduced to overcome this limitation.
\subsection{Lebesgue integration}
Lebesgue integration can be defined on arbitrary sets,
and allows taking limits of sequences of functions under integration,
which is necessary in Fourier analysis, for example.

In this section, we follow \textcite{rudin2006real} for the most part.
Refer to \textcite{gallierncharm} for a more general approach
which allow functions taking value on arbitrary (possibly infinite-dimensional)
vector spaces.

We begin by defining the crucial concept of \emph{measure}.
\begin{definition}[measure]
  A collection $\Sigma$ of subsets of a set $X$ is a \emph{$\sigma$-algebra}
  if it contains $X$ and is closed under complementation and countable unions.
  We call the tuple $(X, \Sigma)$ a \emph{measurable space},
  and the subsets in $\Sigma$ are \emph{measurable sets}.
  A function \fun{f}{X}{Y} is \emph{measurable} if
  the preimage of every measurable set in $Y$ is in $\Sigma$.
  A \emph{measure}  is a function \fun{\mu}{\Sigma}{[0, \infty]}
  which is \emph{countably additive},
  \begin{align}
    \mu\left( \bigcup_{i=0}^\infty A_i \right) = \sum_{i=0}^{\infty} \mu(A_i) \label{eq:additivity}
  \end{align}
  for a disjoint collection of $A_i \in \Sigma$.
  The tuple $(X, \Sigma, \mu)$ is called a \emph{measure space}.
\end{definition}
\begin{example}
  On the real line $\R$, we define $\mathcal{B}(\R)$ as the smallest $\sigma$-algebra containing every open interval.
  This is known as the $\sigma$-algebra of Borel sets, or the Borel algebra.
  Then \fun{\mu}{\mathcal{B}(\R)}{[0, \infty]} defined such that
  $\mu((a,b]) = b-a$ is a measure in $(\R, \mathcal{B}(\R))$;
  it is usually called the Borel measure.
\end{example}

Carathéodory's theorem allows the construction of measures and measure spaces from an outer measure.
\begin{theorem}[Carathéodory]
  \label{thm:caratheodoty}
  An \emph{outer measure} $\mu^*$ on a set $X$ is a function \fun{\mu^*}{2^X}{[0, \infty]}
  such that (i) $\mu^*(\emptyset)=0$, (ii) if $A \subseteq B$, $\mu^*(A) \le \mu^*(B)$ and (iii)
  \begin{align}
    \mu^*\left( \bigcup_{i=0}^\infty A_i \right) \le \sum_{i=0}^{\infty} \mu^*(A_i).
    \label{eq:subadditivity}
  \end{align}
  Note that \cref{eq:subadditivity} is a relaxation of \cref{eq:additivity}, called \emph{subadditivity.}
  We can construct an outer measure on $X$ as
  \begin{align}
    \mu^*(A) = \inf \left\{\sum_{n=0}^\infty \lambda(I_n) \mid A \subseteq \bigcup_{n=0}^\infty I_n \right\}
    \label{eq:outer}
  \end{align}
  where $\lambda$ is any positive function with $\lambda(\emptyset)=0$
  and there is a family $\{I_n\}$ of subsets of $X$ that contains the empty set and
  covers any subset $A \subseteq X$.
  Now consider the family of subsets
  \begin{align*}
    \Sigma = \{A \in 2^X \mid \mu^*(A) = \mu^*(E \cap A) + \mu^*(E \cap (X-A)),\, \forall E \in 2^X\}.
  \end{align*}
  Then $\Sigma$ is a $\sigma$-algebra and the restriction $\mu$ of $\mu^*$ to $\Sigma$ is a measure,
  so $(X, \Sigma, \mu)$ is a measure space.
\end{theorem}

\begin{example}
  Let $\mu^*$ be an outer measure constructed as in \cref{eq:outer}
  where $\{I_n\}$ is the set of all open intervals in $\R$
  and $\lambda((a,b)) = b-a$.
  By applying \cref{thm:caratheodoty} to $\mu^*$ we obtain
  the Lebesgue measure $\mu_L$,
  and the $\sigma$-algebra of Lebesgue measurable sets $\mathcal{L}(\R)$.
  It can be shown that $\mathcal{B}(\R) \subset \mathcal{L}(\R)$;
  this extends to $\R^n$.
\end{example}

Equipped with the notion of measures and measurable functions,
we can define the Lebesgue integral.
\begin{definition}[Lebesgue integral]
  Let $(X, \Sigma, \mu)$ be a measure space.
  We define the \emph{characteristic function} $\chi_A$ of a measurable set $A$ as the
  indicator function \fun{\chi_A}{X}{\{0,1\}} that is 1 when $x \in A$ and 0 otherwise.
  A \emph{simple function} is a function $s$ on $X$ whose range consist only of finitely
  many distinct values;
  formally, $s(x) = \sum_{i=0}^n \alpha_i\chi_{A_i}(x)$ where $\{\alpha_i\}$
  is the set of distinct values.
  We define the integral of a measurable simple function over a set $E \in \Sigma$ as
  \begin{align*}
    \int\limits_E s\,d\mu = \sum_{i=0}^n\alpha_i \mu(A_i \cap E).
  \end{align*}
  We call a function $f$ \emph{positive} when $f(x) \ge 0$ for all $x$,
  and say that $f \le k$ when $k-f$ is positive.
  For a measurable positive function \fun{f}{X}{[0, \infty]} we define the \emph{Lebesgue integral} as
  \begin{align*}
    \int\limits_E f\,d\mu = \sup \int\limits_E s\,d\mu,
  \end{align*}
  where the supremum is over all simple functions $s$ such that $0 \le s \le f$.
\end{definition}
The Lebesgue integral is easily extended to complex valued functions \fun{f}{X}{\C}
by noting that we can write $f = u^{+} - u^{-} + i(v^{+} - v^{-})$
for positive functions $u^{+}$, $u^{-}$, $v^{+}$, $v^{-}$;
the integral is then obtained by linearity.

Intuitively, while the Riemann integral partitions the domain of $f$ to compute
the integral, the Lebesgue integral partitions its range.
This is the key to enable integration over more general domains.
\begin{example}
  Consider again the measure space $(\R, \mathcal{B}(\R), \mu)$,
  and the indicator function for the rational numbers \fun{f}{\R}{\{0, 1\}},
  $f(x)=1$ if $x \in \Q$ and $f(x)=0$ otherwise.
  The function is not Riemann-integrable since here is no interval where it is continuous.
  However it is a simple function that takes the value 1 on a set of measure zero
  (since $\Q$ is countable), and 0 elsewhere.
  Hence, $f$ is Lebesgue integrable and its integral is zero on any interval.
\end{example}

\subsection{The Haar measure}
The Lebesgue integral allows integration on arbitrary sets,
when they are given the structure of a measure space.
The Haar measure gives such structure to locally compact groups.

\begin{theorem}[Haar measure]
  Consider a locally compact Hausdorff topological group $G$,
  and the Borel $\sigma$-algebra $\mathcal{B}(G)$ generated by its open subsets.
  There exists a unique measure $\mu$, up to a multiplicative constant, such that
  $\mu$ is left-invariant, i.e., $\mu(gE) = \mu(E)$ for all $E \in \mathcal{B}(G)$ and $g \in G$.
  Furthermore, $\mu$ is $\sigma$-regular,
  \begin{align*}
    \mu(E) &= \inf \{\mu(U) \mid E \subseteq U,\, U \text{ open}  \}, \\
    \mu(E) &= \inf \{\mu(K) \mid K \subseteq E,\, K \text{ compact} \}.
  \end{align*}
  The measure $\mu$ defined as such is called the \emph{left Haar measure}.
  We define the \emph{right Haar measure} analogously;
  both measures are not necessarily equal.
\end{theorem}
It can be shown that $\mu(U) > 0$ for any non-empty open $U \in \mathcal{B}(G)$ and
$\mu(K) < \infty$ for any compact $K \in \mathcal{B}(G)$.

The construction idea is to define the measure of a subset $K \in \mathcal{B}(G)$
as the number of left-translations of a small $U \in  \mathcal{B}(G)$ necessary to cover $K$.
It is made precise by taking appropriate limits and enforcing measure properties.

Now define the left action operator $\lambda_u(f)(g) = f(u^{-1}g)$.
Given a left Haar measure $\mu$, its left invariance implies
\begin{align}
  \int \lambda_s(f)\,d\mu = \int f \, d\lambda_s(\mu) = \int f d\mu
\end{align}
for any \fun{f}{G}{\C} and $s \in G$.
We write $d\mu(g) = dg$ to simplify the notation;
then the relation $\int_G f(s^{-1}g) \,dg = \int_G f(g) \,dg$
gives a variable substitution formula that appears in many proofs.
For functions on the line, this translates to the usual
$\int_{-\infty}^{\infty} f(x-y)\, dx = \int_{-\infty}^{\infty} f(x)\, dx$,
where the Lebesgue measure is also a Haar measure.
\begin{example}
  Consider again the group $G$ of unitary complex numbers of the form $g_\theta = e^{i\theta}$,
  for $-\pi \le \theta < \pi$, and the function \fun{\lambda}{G}{\R} such that $\lambda(e^{i\theta}) = \theta$.
  We obtain the Haar measure from the Lebesgue measure $\mu_L$ on $\R$ as
  $\mu(A) = \mu_L(\lambda(A))$; it can be shown to be left-invariant.
\end{example}
\begin{example}
  For the group $\GL(n, \R)$, the Haar measure is given by $dA / |det(A)|^n$,
  where $dA$ is the Lebesgue measure on $R^{n^2}$.
\end{example}
\noindent On a Lie group of dimension $n$, we can construct an alternating $n$-form on the tangent space at the
identity and transport it to the tangent space at any point using left group actions.
The result is a left-invariant differential $n$-form (volume form) on the group that induces the left Haar measure \cite{hall2015lie}.

Next, we introduce modular functions, which are useful to determine some group properties.
\begin{definition}[modular function]
Let $\mu$ be a left Haar measure on a group $G$,
and define the right action operator $\rho_s(f)(g) = f(gs)$.
It follows that $\rho_s(\mu)$ is also a left Haar measure and
since the left Haar measure is unique up to scalar multiplication,
we have $\rho_s(\mu) = \Delta(s)\mu$ for \fun{\Delta}{G}{(0, \infty]}.
We call the function $\Delta$ a \emph{modular function}.
If $\Delta(s) = 1$ for all $s \in G$, we call $G$ \emph{unimodular}.
\end{definition}
\noindent In particular, a left Haar measure is also a right Haar measure
if and only if the group is unimodular.
Abelian groups are unimodular, and so are compact groups.

Next, we want to obtain measures on homogeneous spaces from measures on groups.
Let $G$ be a locally compact group with a subgroup $H$.
Now consider the homogeneous space $G/H$
where we suppose there is a measure $\gamma$.
We call $\gamma$ \emph{$G$-invariant} if ${\lambda_u(\gamma) = \gamma}$, for all $u \in G$,
where $\lambda_u$ is the left action operator ${\lambda_u(f)(g) = f(u^{-1}g)}$.
The following theorem gives the conditions for the existence of a $G$-invariant measure.
\begin{theorem}
  \label{thm:hmeasure}
  Let $G$ be a locally compact group with a subgroup $H$,
  $\mu$ a left Haar measure on $G$ and $\xi$ a left Haar measure in $H$.
  There is a unique G-invariant measure $\gamma$ on G/H (up to scalar multiplication)
  if and only if the modular function $\Delta_H$ equals the restriction of $\Delta_G$ to $H$.
  We can then write
  \begin{align*}
    \int\limits_G f(u)\,d\mu(u) = \int\limits_{G/H}\int\limits_{H} f(uh)\,d\xi(h)d\gamma(uH),
  \end{align*}
  for any function $f$ of compact support on $G$.
\end{theorem}

\begin{remark}
  When the $\Delta_H$ is not equal to the restriction of $\Delta_G$ to $H$,
  there is a weaker form of invariance in measures, called \emph{quasi-invariance}.
  Quasi-invariant measures on $G/H$ always exist. Refer to \textcite{folland2016course,gallierncharm}
  for details.
\end{remark}

\section{Harmonic analysis}
\label{sec:harm}
Recall the Fourier series expansion of a periodic function $f$
\begin{align*}
  f(\theta) &= \sum_{m \in \Z} \hat{f}(m) e^{i  m \theta}, \\
  (\mathcal{F}f)(m) &= \hat{f}(m) = \frac{1}{2\pi} \int\limits_{-\pi}^{\pi} f(\theta)e^{-i m \theta}\, d\theta.
\end{align*}
A periodic scalar-valued function $f$ can be seen as a function on the circle \fun{f}{S^1}{\R}.
The expansion in Fourier series is a decomposition in the basis $\{e^{im\theta}\}$ for $m \in \Z$
of the space of square-integrable functions on the circle, $L^2(S^1)$.
Fourier analysis has numerous applications in signal processing, differential equations and number theory.
Most important for our purposes is the convolution theorem,
\begin{align}
  \mathcal{F}(f * k)(m) = (\mathcal{F}f)(m)(\mathcal{F}k)(m) = \hat{f}(m)\hat{k}(m),
\end{align}
which states that convolution in the spatial domain corresponds to multiplication in the spectral domain.
This has immense practical implications for efficient computation of convolutions,
thanks to the \fft\ algorithm.

In this section, we generalize these concepts to functions on compact groups.
We consider a compact group $G$,
and the Hilbert space $L^2(G)$ of square integrable functions on $G$.
Integrals on compact groups are well defined as discussed in \cref{sec:haar}.
We state the Peter-Weyl theorem, which gives an
orthonormal basis for $L^2(G)$ constructed from irreducible representations of $G$.
This paves the way to harmonic analysis on compact groups,
which we demonstrate by generalizing the Fourier transform and convolution theorem.
Again these have important practical applications
and are used to compute group convolutions in recent equivariant neural networks.
Finally, we show how the theory applies to homogeneous spaces of compact groups.
\subsection{The Peter-Weyl Theorem}
The Peter-Weyl theorem gives an explicit
orthonormal basis for $L^2(G)$, constructed from irreducible representations of a group $G$.
The basis is formed by matrix elements, which we define first.
\begin{definition}[matrix elements]
  \label{def:me}
  Let $\rho$ be a unitary representation of a compact group $G$.
  We denote $\phi_{x,y}(g) = \inner{\rho(g)(x), y}$ the \emph{matrix elements} of $\rho$.
  Note that $\phi_{e_i,e_j}(g)$ is one entry of the matrix $\rho(g)$ when $e_i,\, e_j$ are basis vectors,
  so we define $\rho_{ij}(g) = \phi_{e_i,e_j}(g)$.%
\end{definition}
\begin{theorem}[Peter-Weyl]
  \label{thm:pw}
  Let $G$ be a compact group. We present the theorem in three parts.
  The first relates matrix elements and spaces of functions on $G$.
  The second decomposes representations of $G$,
  and the third gives a basis for $L^2(G)$.
  \paragraph{Part I} The linear span of the set of matrix elements of unirreps
  of $G$ is dense in the space of continuous complex valued functions on $G$, under the uniform norm.
  This implies it is also dense in $L^2(G)$.
  \paragraph{Part II} Let $\hat{G}$ be the set of equivalence classes of unirreps of $G$.
  For a unirrep $\rho$ of $G$,
  we denote its representation space by $H_\rho$ where $\dim H_\rho = d_\rho$,
  and its equivalence class by $[\rho] \in \hat{G}$.
  If $\pi$ is a (reducible) unitary representation of $G$,
  it splits in the orthogonal direct sum $H_\pi = \bigoplus_{[\rho] \in \hat{G}} M_\rho $,
  where $M_\rho$ is the largest subspace where $\pi$ is equivalent to $\rho$.
  Each $M_\rho$ splits in equivalent irreducible subspaces $M_\rho = \bigoplus_{i=1}^{n} H_\rho$,
  where $n$ is the \emph{multiplicity} of $[\rho]$ in $\pi$.
  \paragraph{Part III}
  Let $\varepsilon_\rho$ be the linear span of the matrix elements of $\rho$ for $[\rho] \in \hat{G}$.
  $L^2(G)$ can be decomposed as $L^2(G) = \bigoplus_{[\rho] \in \hat{G}} \varepsilon_\rho$.
  If $\pi$ is a regular representation on $L^2(G)$, the multiplicity of $[\rho] \in \hat{G}$ in $\pi$ is $d_\rho$.
  An orthonormal basis of $L^2(G)$ is
  \begin{equation*}
    \{ \sqrt{d_\rho}\rho_{ij} \mid 1 \le i,\, j \le d_\rho,\, [\rho] \in \hat{G} \}
  \end{equation*}
  where $\rho_{ij}$ is as in \cref{def:me}.
  Constructing the basis involves choosing a representative per equivalence class.
\end{theorem}
\begin{example}
  The $\SO(3)$ irreducible representations $\rho^\ell$ can be written as $2\ell+1 \times 2\ell+1$ matrices for $\ell \in \N$, with entries \fun{\rho_{ij}^\ell}{\SO(3)}{\C},
  \begin{align*}
    \rho^{0} = (\rho_{0,0}^{0}), &&
                                   \rho^{1} = \threebythree
                                   {\rho_{-1,-1}^{1}}{\rho_{-1,0}^{1}}{\rho_{-1,1}^{1}}
                                   {\rho_{0,-1}^{1}}{\rho_{0,0}^{1}}{\rho_{0,1}^{1}}
                                   {\rho_{1,-1}^{1}}{\rho_{1,0}^{1}}{\rho_{1,1}^{1}}, &&
                                                                                         \rho^{2} = \cdots\, ,
  \end{align*}
  and the matrix elements $\rho_{i,j}^{\ell}$ form a basis for $L^2(\SO(3))$.
  We will derive these elements in \cref{sec:su2}.
\end{example}
For simplicity, we avoided introducing Hilbert algebras, ideals,
and the interesting connection between representations of groups and of algebras.
We refer the reader to \textcite{gallierncharm} for a complete description
of the Peter-Weyl theorem, with proofs.
\subsection{Fourier analysis on compact groups}
Part III of \cref{thm:pw} gives an orthonormal basis for $L^2(G)$,
so for any $f \in L^2(G)$ we can write,
\begin{align}
  f(g) &= \sum_{[\rho] \in \hat{G}}\sum_{i,j=1}^{d_\rho} c_{ij}^\rho \rho_{ij}(g), \label{eq:fwcoord} \\
  c_{ij}^\rho &= d_\rho \int\limits_{g \in G} f(g)\overline{\rho_{ij}(g)}\, dg, \label{eq:decomp}
\end{align}
where \cref{eq:decomp} is the inner product in $L^2(G)$,
the matrix elements $\rho_{ij}$ are as in \cref{def:me},
and the coefficients $c_{ij}^\rho$ absorb an extra $\sqrt{d_\rho}$ for simplification.

Now we define the Fourier transform of $f \in L^2(G)$ as a function on $\hat{G}$
whose values are on $\GL(H_\rho)$ for an input $[\rho]$.
\begin{equation}
  \hat{f}(\rho) = \mathcal{F}(f)([\rho]) = \int\limits_{g \in G} f(g) \rho(g)^{*}\, dg, \label{eq:ft}
\end{equation}
where $\rho$ is the representative of $[\rho]$,
$^{*}$ indicates the conjugate transpose,
and we introduce $\hat{f}(\rho)$ to shorten notation.
It is easy to see that the element $i$, $j$ of $\hat{f}(\rho)$
corresponds to $\frac{c_{ji}^\rho}{d_\rho}$ as defined in \cref{eq:decomp}, and
\begin{equation*}
  \sum_{i,j=1}^{d_\rho} c_{ij}^\rho \rho_{ij}(g) =
  \sum_{i,j=1}^{d_\rho} d_\rho \hat{f}(\rho)_{ji}\rho_{ij}(g) =
  d_\rho \tr(\hat{f}(\rho) \rho(g)).
\end{equation*}
Applying this result to \cref{eq:fwcoord} yields the Fourier inversion formula,
\begin{equation}
  f(g) = \sum_{[\rho] \in \hat{G}} d_\rho \tr(\hat{f}(\rho) \rho(g)). \label{eq:fi}
\end{equation}
\begin{remark}
\Cref{eq:ft,eq:fi} give the Fourier transform and inverse
independently of the choice of a basis, in contrast with \cref{eq:fwcoord,eq:decomp}.
\end{remark}
\begin{remark}
We are not discussing convergence here;
refer to \textcite{folland2016course,gallierncharm} for details.
\end{remark}
\begin{example}
  Consider the multiplicative group of complex numbers of the form $e^{i\theta}$,
  identified with the planar rotation group $\SO(2)$.
  The unirreps of this group on $\C$ are given by $\rho_n(e^{i\theta}) = e^{in\theta}$,
  for $n \in \Z$.
  Since they are one dimensional, they are also the matrix elements and hence form an orthonormal basis
  for $L^2(\SO(2))$.
  We can index $\rho_n$ by $n$ and $\SO(2)$ by $\theta$,
  and write the Fourier transform and inverse on $L^2(\SO(2))$ as
  \begin{align}
    \hat{f}(n) &= \int\limits_{0}^{2\pi} f(\theta) e^{-in\theta} \, \frac{d\theta}{2\pi}, \\
    f(\theta) &= \sum_{n=-\infty}^{\infty} \hat{f}(n) e^{in\theta},
  \end{align}
  which are the familiar formulas for the Fourier series of periodic functions.

  This simple example shows how the Fourier analysis of periodic functions on the line fit in the theory described.
  See \cref{sec:su2} for a more complete example with a non-commutative group.
\end{example}
\subsection{Convolution theorem on compact groups}
Given the existence of the left Haar measure as discussed in \cref{sec:haar},
we define the convolution between functions \fun{f,k}{G}{\C} on a group $G$ as
\begin{equation}
  (f * k)(g) = \int\limits_{u \in G} f(u)k(u^{-1}g)\, du = \int\limits_{u \in G} f(gu)k(u^{-1})\, du.
  \label{eq:gconv}
\end{equation}

A simple change of variables leveraging the measure left-invariance
shows that group convolution is equivariant,
\begin{align*}
  (\lambda_u f * k)(g)
  &= \int\limits_{v \in G} f(u^{-1}v) k(v^{-1}g) \, dv\\
  &= \int\limits_{w \in G} f(w) k((uw)^{-1}g)  \, dw && (v \mapsto uw) \\
  &= \int\limits_{w \in G} f(w) k(w^{-1}u^{-1}g)  \, dw\\
  &= (f * k)(u^{-1}g) \\
  &= (\lambda_u (f * k))(g).
\end{align*}

\begin{theorem}[Convolution theorem]
  \label{thm:conv}
  Let $f$ and $k$ be square integrable functions on a compact group $G$ ($f,\,g \in L^2(G)$).
  The Fourier transform of the convolution $f*k$ is $(\widehat{f * k})(\rho) = \hat{k}(\rho) \hat{f}(\rho)$.
\end{theorem}
\begin{proof}
Let us compute the Fourier transform of $f * k$ using \cref{eq:ft},
\begin{align*}
  (\widehat{f * k})(\rho) &= \int\limits_{g \in G} \left(\int\limits_{u \in G} f(u)k(u^{-1}g)\, du\right) \rho(g)^{*}\, dg \\
    &= \int\limits_{u \in G} \int\limits_{g \in G} f(u)k(u^{-1}g) \rho(g)^{*} \,dg\, du \\
    &= \int\limits_{u \in G} \int\limits_{v \in G} f(u)k(v) \rho(uv)^{*} \,dv\, du && (v = u^{-1}g) \\
    &= \int\limits_{v \in G} k(v) \rho(v)^{*} \int\limits_{u \in G} f(u)\rho(u)^{*} \,dv\, du
    && \text{(homomorphism, reorder)} \\
    &= \hat{k}(\rho) \hat{f}(\rho).
\end{align*}
\end{proof}
\begin{remark}
  There is an analogous cross-correlation theorem that we prove in the same way.
  We define the group cross-correlation as
  \[(f \star k)(g) = \int\limits_{u \in G} f(u)  k(g^{-1} u) \, du \]
  and follow the same steps as before, obtaining
\begin{align*}
  (\widehat{f \star k})(\rho)
  &= \int\limits_{v \in G} k(v) \rho(v^{-1})^{*} \int\limits_{u \in G} f(u)\rho(u)^{*} \,dv\, du.
\end{align*}
Note that the only difference is the term $v^{-1}$.
Since $\hat{k}(\rho) = \int_{v\in G}k(v)\rho(v)^*$,
assuming real-valued $k$ we have
$\hat{k}(\rho)^* = \int_{v\in G}k(v)\rho(v^{-1})^*$ and
\begin{align}
  (\widehat{f \star k})(\rho) = \hat{k}(\rho)^*\hat{f}(\rho).
  \label{eq:groupcorrspectral}
\end{align}
\end{remark}
\noindent This shows that the Fourier transform of the compact group convolution
is the matrix product of the Fourier transforms of each input.
It generalizes the convolution theorem on the circle,
which says that the Fourier transform of the convolution is the scalar multiplication
of the inputs Fourier transforms.

The convolution theorem is fundamental for the efficient computation of convolutions,
since the \fft\ can be generalized to compact groups \cite{driscoll1994computing,kostelec2008ffts}.
Furthermore, the spectral computation avoids interpolation errors and
extra computational cost caused by the lack of regular grids for arbitrary groups.
\subsection{Examples: $\SL(2, \C)$, $\SU(2)$ and $\SO(3)$}
\label{sec:su2}
Now we find expressions for the matrix elements of
representations of $\SL(2, \C)$, $\SU(2)$ and $\SO(3)$,
which allow computing the Fourier transforms and convolutions on these groups.
We follow one of the approaches by \textcite{vilenkin1978special},
also used by \textcite{dieudonn1980special,gurarie2007symmetries}.

The strategy is to first find the matrix elements for irreps of $\SL(2, \C)$,
then restrict them to $\SU(2)$ and $\SO(3)$.
\subsubsection{Representations of $\SL(2, \C)$}
The special linear group $\SL(2, \C)$ consists of $2\times 2$ complex matrices with determinant 1,
\begin{align}
  g = \twobytwo{a}{c}{b}{d} \label{eq:gsl2}
\end{align}
where $ad - bc = 1$.

Now consider the space $V_{\ell}$ of homogeneous polynomials of degree $2\ell$ in two complex variables,
where $\ell$ is integer or half-integer,
\begin{align*}
  x = \twobyone{x_1}{x_2}, && P_\ell(x) = P_\ell(x_1, x_2) = \sum_{i=-\l}^\l\alpha_ix_1^{\ell-i}x_2^{\ell+i}.
\end{align*}
We define \fun{\pi_\ell}{\SL(2,\C)}{\GL(V_{\ell})} as
\begin{align}
  (\pi_\ell(g)P_\ell)(x) = P_\ell(g^{-1}x), \label{eq:reprsl}
\end{align}
which is linear and a group homomorphism.
Hence, $\pi_\ell$ is a representation of $\SL(2, \C)$ on $V_{\ell}$
(of dimension $2\ell+1$).
Furthermore, it can be shown that these are irreducible,
and in fact these are the only irreps of $\SL(2, \C)$ and $\SU(2)$, up to equivalence.

Now let us derive expressions for the matrix elements.
Consider the polynomial in one variable
$Q_\ell(x) = P_\ell(x, 1) = \sum_{i=-\l}^\l\alpha_ix^{\ell-i}$, of degree $2\ell$.
Writing $P_\ell$ in terms of $Q_\ell$ yields
\begin{align}
  P_\ell(x_1, x_2) = x_2^{2\ell}Q_\ell(x_1/x_2). \label{eq:pq}
\end{align}
We denote $H_\ell$ the space of all polynomials $Q_\ell$ (of degree $2\ell$) for $\ell \ge 0$.
We rewrite \cref{eq:reprsl} for $g$ as in \cref{eq:gsl2} where $g^{-1}=\twobytwo{d}{-c}{-b}{a}$,
\begin{align}
  (\pi_\ell(g)P_\ell)(x_1, x_2) = P_\ell(dx_1-cx_2, -bx_1 + ax_2), \label{eq:reprp}
\end{align}
and define $\rho_\ell$ as the application of $\pi_\ell$ to $Q_\ell\in H_\ell$ using \cref{eq:reprp,eq:pq}
\begin{align}
  (\rho_\ell(g)Q_\ell)(x) = (-bx + a)^{2\ell}Q_\ell\left(\frac{dx-c}{-bx + a}\right). \label{eq:repq}
\end{align}

The monomials $x^{\ell-m}$ for $-\ell \le m \le \ell$ are a basis of $H_\ell$.
Now consider the inner product on $H_\ell$ defined by
\begin{align}
  \inner{x^{\ell-m}, x^{\ell-n}} &= 0, && m \neq n \label{eq:innerq1} \\
  \inner{x^{\ell-m}, x^{\ell-m}} &= (\ell-m)!(\ell+m)!, \label{eq:innerq2}
\end{align}
which is adapted from an inner product (sometimes called the Bombieri scalar product) on $V_\ell$
\[ \inner{x^{\ell+m}y^{\ell-m},x^{\ell+m}y^{\ell-m}} = (\ell+m)!(\ell-m)!/(2\ell)!. \]

It turns out the representation $\rho_\ell$ defined as in \cref{eq:repq}
is unitary under the inner product defined by \cref{eq:innerq1,eq:innerq2}.
The following is an orthonormal basis $\{\psi_m\}$ for $H_\ell$ with this inner product
\begin{align*}
  \psi_m(x) = \frac{x^{\ell-m}}{\sqrt{(\ell-m)!(\ell+m)!}}.
\end{align*}
The element at position $(m, n)$%
\footnote{Not the conventional way of indexing since $-\ell \le m, n \le \ell$,
but convenient in our notation.}
of the matrix for $\rho_\ell(g)$ under this basis is
\begin{align}
\rho_{\ell}^{mn}(g) = \inner{\rho_{\ell}(g)(\psi_n), \psi_m}. \label{eq:rhoij}
\end{align}
According to \cref{eq:repq}, $\rho_\ell$ acts on $Q(x) = x^{\ell-n}$ as
\begin{align*}
\rho_\ell(g)x^{\ell-n} = (-bx + a)^{\ell+n}(dx-c)^{\ell-n}
\end{align*}
for $g$ as in \cref{eq:gsl2}.
We substitute it in \cref{eq:rhoij} to obtain
\begin{align}
  \rho_{\ell}^{mn}(g) = \frac{\inner{(-bx + a)^{\ell+n}(dx-c)^{\ell-n}, x^{\ell-m}}}
  {\sqrt{(\ell-n)!(\ell+n)!(\ell-m)!(\ell+m)!}}. \label{eq:rhoijlong}
\end{align}
Observe that $\inner{Q(x),x^{\ell-m}}$ for some polynomial $Q(x)$
is the coefficient of $x^{\ell-m}$ in $Q(x)$ multiplied by $(\ell-m)!(\ell+m)!$,
according to \cref{eq:innerq2}.
Recall that the Taylor formula for a function $Q(x)$ around $x=0$ is
$Q(x) = \sum_{n=0}^\infty \frac{d^n}{dx^n}\frac{x^n}{n!}$.
We apply it to obtain the coefficient of $x^{\ell-m}$ in \cref{eq:rhoijlong},
\begin{align}
  \rho_{\ell}^{mn}(g) = \sqrt{\frac{(\ell+m)!}{(\ell-n)!(\ell+n)!(\ell-m)!}}
  \frac{d^{\ell-m}}{dx^{\ell-m}}\left((-bx + a)^{\ell+n}(dx-c)^{\ell-n}\right)\Big|_{x=0},
  \label{eq:mesl2}
\end{align}
with $g=\twobytwo{a}{c}{b}{d}$ as usual.
Substituting $z=b(dx-c)$ and using that $ad-bc=1$ yields
\begin{align}
  \rho_{\ell}^{mn}(g) = \sqrt{\frac{(\ell+m)!}{(\ell-n)!(\ell+n)!(\ell-m)!}}
  \frac{b^{n-m}}{d^{n+m}}
  \frac{d^{\ell-m}}{dz^{\ell-m}}\left((1-z)^{\ell+n}z^{\ell-n}\right)\Big|_{z=-bc}.
  \label{eq:mesl2z}
\end{align}
This is a general formula for matrix elements of the unirreps
of $\SL(2, \C)$, which generate an orthonormal basis of $L^2(\SL(2, \C))$
as stated by the Peter-Weyl theorem.
\subsubsection{Representations of $\SU(2)$}
We now restrict the $\SL(2, \C)$ representations to $g \in \SU(2)$,
the group of $2\times 2$ unitary matrices with determinant \SI{1}{}.
So for $g \in \SU(2)$ we have $g^{*}g = gg^* = I$, which implies
\begin{align}
  g = \twobytwo{a}{-\overline{b}}{b}{\overline{a}} \label{eq:gsu2}
\end{align}
where $a\overline{a} + b\overline{b} = 1$, and the bar denotes the complex conjugate.
It follows that $\SU(2) < \SL(2,\C)$.
We can factor $g \in \SU(2)$ as
\begin{align}
  g_{\alpha\beta\gamma} =
  \twobytwo{e^{-i\alpha/2}}{0}{0}{e^{i\alpha/2}}
  \twobytwo{\cos(\beta/2)}{-\sin(\beta/2)}{\sin(\beta/2)}{\cos(\beta/2)}
  \twobytwo{e^{-i\gamma/2}}{0}{0}{e^{i\gamma/2}},
  \label{eq:factorsu2}
\end{align}
where $\alpha$, $\beta$ and $\gamma$ are ZYZ Euler angles,
$0 \le \alpha < 2\pi$, $0 \le \beta < \pi$ and $-2\pi \le \gamma < 2\pi$.
Now consider representations \fun{\rho_\ell}{\SU(2)}{\GL(H_\ell)},
which are a special case of the representations of $\SL(2, \C)$,
and hence inherit their properties.
Since $\rho_\ell$ is a group homomorphism,
\begin{align}
  \rho_\ell(g_{\alpha \beta \gamma}) =
  \rho_\ell(g_{\alpha 0 0})
  \rho_\ell(g_{0 \beta 0})
  \rho_\ell(g_{0 0 \gamma}).
  \label{eq:factorrho}
\end{align}
Since $\rho(g_{\alpha 0 0})$ corresponds to $a=e^{-i\alpha/2}$, $d=\overline{a}$, and $b=c=0$
in \cref{eq:repq}, we find that
$\rho_\ell(g_{\alpha 0 0})(\psi_m)= e^{-i\alpha m} \psi_m$,
which implies that only the diagonal elements of $\rho_\ell(g_{\alpha 0 0})$
are nonzero; they are
\begin{align}
\rho_\ell^{mm}(g_{\alpha 0 0}) = e^{-i\alpha m}. \label{eq:rholmm}
\end{align}
The expression for $\rho_\ell(g_{0 0 \gamma})$ is analogous.
The middle factor in \cref{eq:factorrho} is multiplied by diagonal matrices on both sides,
so we write the matrix elements
\begin{align*}
\rho_\ell^{mn}(g_{\alpha \beta \gamma}) = e^{-i(m\alpha + n\gamma)}\rho_\ell^{mn}(g_{0 \beta 0}).
\end{align*}
To compute $\rho_\ell^{mn}(g_{0 \beta 0})$,
we apply
\begin{align*}
a=d=\cos(\beta/2) \text{, } b=\sin(\beta/2) \text{ and } c=-\sin(\beta/2)
\end{align*}
to \cref{eq:mesl2z},
and note that the derivative is evaluated at $z=-bc=\sin^2(\beta/2)$.
We define $P_{mn}^\ell(\cos\beta) = \rho_\ell^{mn}(g_{0 \beta 0})$,
and make the substitution
\[z \mapsto \frac{1-x}{2},\]
where the derivative is now evaluated at $x=-2\sin^2(\beta/2)+1 = \cos\beta$.
Then $b=\sqrt{(1-x)/2}$ and $d=\sqrt{(1+x)/2}$. We have,
\begin{align}
  P_{mn}^\ell(x)
  &=  c_{mn}^\ell \frac{\left(\frac{1-x}{2}\right)^{\frac{n-m}{2}}}
    {\left(\frac{1+x}{2}\right)^{\frac{n+m}{2}}}
    \frac{(-1)^{\ell-m}}{2^{m-\ell}}
  \frac{d^{\ell-m}}{dx^{\ell-m}}
    \left(\left(\frac{1+x}{2}\right)^{\ell+n}\left(\frac{1-x}{2}\right)^{\ell-n}\right) \nonumber \\
  &=  c_{mn}^\ell \frac{(-1)^{\ell-m}}{2^{\ell}}
    \frac{(1-x)^{\frac{n-m}{2}}}{(1+x)^{\frac{n+m}{2}}}
    \frac{d^{\ell-m}}{dx^{\ell-m}}
    \left((1+x)^{\ell+n}(1-x)^{\ell-n}\right). \label{eq:mesu2}
\end{align}
where
\begin{align*}
    c_{mn}^\ell = \sqrt{\frac{(\ell+m)!}{(\ell-n)!(\ell+n)!(\ell-m)!}},
\end{align*}
and
\begin{align}
  \rho_\ell^{mn}(g_{\alpha \beta \gamma}) = e^{-i(m\alpha + n\gamma)}P_{mn}^\ell(\cos\beta),
  \label{eq:mesu2general}
\end{align}
which is a general formula for matrix elements of $\SU(2)$ unirreps.
The matrices formed with the $\rho_\ell^{mn}$ and $P_{mn}^\ell$ are
also known as a Wigner-D and Wigner-d matrices, respectively.
\subsubsection{Representations of $\SO(3)$}
$\SU(2)$ is isomorphic to the group of unit quaternions, hence a double cover of $\SO(3)$,
which is easily verifiable by noting that every rotation in $\SO(3)$ can be written as two
different quaternions $q$ and $-q$.
We have ${\SO(3) \cong \SU(2)/\{I, -I\}}$.
The representations of $\SO(3)$ are then those representations of $\SU(2)$ where
$\rho_\ell(I) = \rho_\ell(-I)$.
By substituting $b=c=0$ in \cref{eq:mesl2},
we see that the only nonzero terms outside the square root
occur when $n=m$, yielding diagonal matrices with entries
proportional to $a^{\ell + m}d^{\ell - m}$,
\begin{align}
  \rho_{\ell}^{mm}\left(\twobytwo{a}{0}{0}{d}\right)
  &= \frac{1}{(\ell-m)!}a^{\ell + m}d^{\ell - m}.
\end{align}
Recall that for $\SU(2)$ representations, $\ell$ can be integer or half integer.
For $a=d=1$ the expression reduces to $1/(\ell-m)!$
while for $a=d=-1$ it reduces to $(-1)^{2\ell}/(\ell-m)!$,
from where we conclude that $\rho_\ell(I) = \rho_\ell(-I)$ only when $\ell$ is integer.
Therefore, the representations of $\SO(3)$ are also given by \cref{eq:mesu2general},
but with $\ell$ taking only integer values.

\subsubsection{Relation with special functions}
The Jacobi polynomials generalize the Gegenbauer, Legendre, and Chebyshev polynomials,
and thus give origin to several special functions.
One way to represent the Jacobi polynomials is via the Rodrigues' formula%
\footnote{Not to be confused with the Rodrigues' rotation formula.}
\begin{align*}
  P_n^{(\alpha,\beta)}(z) = \frac{(-1)^n}{2^n n!}
  (1-z)^{-\alpha} (1+z)^{-\beta}
  \frac{d^n}{dz^n}
  \left( (1-z)^{\alpha+n} (1+z)^{\beta+n} \right).
\end{align*}
Note how it is tightly related to our expression for the matrix elements in \cref{eq:mesu2},
showing how the special functions arise in the study of group representations.

By setting $m=n=0$ and $\ell$ integer in \cref{eq:mesu2}, we get
\begin{align}
  P_{00}^\ell(x) = \frac{(-1)^{\ell}}{2^{\ell}\ell !}
  \frac{d^{\ell}}{dz^{\ell}}
  (1-x^2)^{\ell},
\end{align}
the Legendre polynomials, which describe the zonal spherical harmonics.

The associated Legendre polynomials can be written as
\begin{align*}
  P_m^\ell(x) = \frac{(-1)^{\ell + m}}{2^\ell \ell!} (1-x^2)^{m/2} \frac{d^{\ell+m}}{dx^{\ell+m}} (1-x^2)^\ell.
\end{align*}
By setting $\ell$ integer and $n=0$ in \cref{eq:mesu2},
we can relate $P_{m,n}^\ell$ with the associated Legendre polynomials,
\begin{align*}
  P_{-m,0}^\ell(x)
  &= \frac{(-1)^{\ell+m}}{2^\ell \ell!} \sqrt{\frac{(\ell-m)!}{(\ell+m)!}}
  (1-x^2)^{m/2}
  \frac{d^{\ell+m}}{dz^{\ell+m}}
  \left((1-x^2)^{\ell}\right) \\
  &= \sqrt{\frac{(\ell-m)!}{(\ell+m)!}} P_m^\ell(x).
\end{align*}
Noting that $P_{m,n}^\ell=P_{-m,-n}^\ell$ we write
\begin{align}
  P_m^\ell(x) = \sqrt{\frac{(\ell+m)!}{(\ell-m)!}} P_{m0}^\ell(x).
  \label{eq:assocleg}
\end{align}

The spherical harmonics are usually defined in terms of the associated Legendre polynomials
\begin{align}
  Y_m^\ell(\theta, \phi) = \sqrt{\frac{(2\ell + 1)}{4\pi} \frac{(\ell-m)!}{(\ell+m)!}}
  P_{m}^\ell(\cos\theta)e^{im\phi}.
  \label{eq:sphharmdef}
\end{align}
Using \cref{eq:mesu2general,eq:assocleg}, we obtain a relation
between the spherical harmonics and the representations $\rho_\ell^{mn}$ ,
\begin{align}
  \rho_\ell^{m0}(g_{\alpha \beta \gamma})
  &= P_{m0}^\ell(\cos\beta) e^{-im\alpha} \nonumber \\
  &= \sqrt{\frac{(\ell-m)!}{(\ell+m)!}} P_{m}^\ell(\cos\beta) e^{-im\alpha} \nonumber \\
  &= \sqrt{\frac{4\pi}{(2\ell+1)}} \overline{Y_m^\ell(\beta, \alpha)}.
  \label{eq:wig2sph}
\end{align}
With this relation, we find an expression for the rotation of spherical harmonics.
Let $g\nu$ be the point obtained by rotating the north pole by $g$.
Since $\rho_{\ell}(g_1g_2) = \rho_{\ell}(g_1) \rho_{\ell}(g_2)$,
\begin{align*}
  \rho_{\ell}^{m0}(g_1g_2) &= \sum_{n=-\ell}^{\ell} \rho_{\ell}^{mn}(g_1) \rho_{\ell}^{n0}(g_2), \\
  \overline{Y_m^\ell(g_1 g_2 \nu)} &= \sum_{n=-\ell}^{\ell} \rho_{\ell}^{mn}(g_1)
                                     \overline{Y_n^\ell(g_2\nu)}.
\end{align*}
Taking conjugates on both sides we arrive at the spherical harmonics rotation formula,
which will be useful in following proofs.
For $x\in S^2$ and $g \in \SO(3)$,
\begin{align}
  \label{eq:sphharmrot}
  Y_m^\ell(g x) &=
\sum_{n=-\ell}^{\ell} \overline{\rho_{\ell}^{mn}(g)} Y_n^\ell(x),
\end{align}
which we write in vector notation as
$Y^\ell(g x) = \overline{\rho_{\ell}(g)} Y^\ell(x).$

\subsection{Fourier analysis on homogeneous spaces}
\label{sec:fourierh}
We now consider functions on the homogeneous space $G/H$ of a compact group $G$ with subgroup $H$;
specifically, consider square integrable functions in $L^2(G/H)$.
Recall that $G/H$ is the set of left cosets and that ${gHh = gH}$ for all $gH \in G/H$ and $h \in H$.
Hence, we can regard functions in $L^2(G/H)$ as the functions in $L^2(G)$ such that $f(gh) = f(g)$
for all $g \in G$ and $h \in H$
(functions that are constant on each coset $gH$ for all $g \in G$).
Using \cref{eq:fwcoord}, we write
$f(g) = \sum_{[\rho] \in \hat{G}}\sum_{i,j=1}^{d_\rho} c_{ij}^\rho \rho_{ij}(g)$,
and expand $f(gh)$ as
\begin{align*}
  f(gh) &= \sum_{[\rho] \in \hat{G}}\sum_{i,j=1}^{d_\rho} c_{ij}^\rho \rho_{ij}(gh) \\
        &= \sum_{[\rho] \in \hat{G}}\sum_{i,j=1}^{d_\rho} c_{ij}^\rho \sum_{k=1}^{d_\rho}\rho_{ik}(g)\rho_{kj}(h) \\
        &= \sum_{[\rho] \in \hat{G}}\sum_{i=1}^{d_\rho}\sum_{k=1}^{d_\rho}\left( \sum_{j=1}^{d_\rho}  c_{ij}^\rho \rho_{kj}(h)\right) \rho_{ik}(g).
\end{align*}
We want $f(g)=f(gh)$, so we compare this expression with the expansion of $f(g)$.
Since the $\rho_{ij}$ are linearly independent, we have
\begin{equation}
  \sum_{j=1}^{d_\rho}  c_{ij}^\rho \rho_{kj}(h) = c_{ik}^{\rho} \label{eq:cij}
\end{equation}
for all $\rho$, $i$, $k$, and $h$.
Now suppose the trivial representation of $H$ has multiplicity $n_\rho \ge 1$ in
the restriction of $\rho$ to $H$.
We can reorder the basis such that the trivial representations appear first.
This implies $\rho_{kj}(h) = \delta_{kj}$ for $j \le n_\rho$ which agrees with $\cref{eq:cij}$.
After reordering, $\rho_{kj}$ integrates to zero for $k > n_\rho$ or $j > n_\rho$,
(only trivial matrix elements integrals are nonzero).
Applying this to both sides of \cref{eq:cij} yields $c_{ik}^\rho = 0$ for $k > n_\rho$,
which implies that any $f \in L^2(G/H)$ can be expanded as
\begin{align}
  f(gH) &= \sum_{[\rho] \in \hat{G}}\sum_{i=1}^{d_\rho}\sum_{j=1}^{n_\rho} c_{ij}^\rho \rho_{ij}(g),
  \label{eq:expandh}
\end{align}
where $n_\rho$ is the multiplicity of the trivial representation of $H$ in $\rho$ (which may be zero).
Only the first $n_\rho$ columns of each $\rho$ are necessary for the Fourier analysis on homogeneous spaces.
In the special case that $n_\rho=1$, only the matrix elements $\rho_{i1}$ will appear.
These are called the \emph{associated spherical functions} \cite{vilenkin1978special}.

When considering functions on the homogeneous space of right cosets, $L^2(H\backslash G)$,
we arrive at similar results where only the first $n_\rho$ rows will appear in the expansion.
When considering functions on the double coset space $L^2(H\backslash G/H)$,
only the first $n_\rho$ rows and columns will appear.
In this last case, when $n_{\rho}=1$, only the matrix elements $\rho_{11}$ appear.
These are called \emph{zonal spherical functions}.
When $n_{\rho} \leq 1$ for every $\rho$, the algebra (with the convolution product)
$L^2(H\backslash G/H)$ is commutative.
\begin{remark}
  The functions just defined are called \emph{spherical} because of the
  special case $G=\SO(3)$ and $H=\SO(2)$ (recall that $S^2 \cong \SO(3)/\SO(2)$).
  These terms apply, however, to any compact group and its homogeneous spaces.
\end{remark}
\begin{remark}
  This discussion generalizes to a locally compact group $G$ (not necessarily compact),
  and compact subgroup $K$, under certain conditions where $(G, K)$ is called a Gelfand pair.
  Refer to \textcite{gallierncharm} for details.
\end{remark}
\subsection{Example: Fourier analysis on $S^2$}
We apply the results of \cref{sec:fourierh} to the group $G=\SO(3)$ and subgroup $H=\SO(2)$,
where the homogeneous space is isomorphic to the sphere $S^2 \cong \SO(3)/\SO(2)$.
Elements of $\SO(3)$ decompose in Euler angles components similarly to \cref{eq:factorsu2}
and by setting $\alpha=\beta=0$ we obtain a subgroup isomorphic
to $\SO(2)$ consisting of rotations around the axis through the poles.
We obtain the restriction of $\SO(3)$ irreps to this subgroup
by setting $\alpha=\beta=0$ for integer $\ell$ in \cref{eq:factorrho},
resulting in $\rho_\ell(g_{0 0 \gamma})$ which is diagonal and defined by
$\rho_\ell^{nn}(g_{0 0 \gamma}) = e^{-i\gamma n}$ (\cref{eq:rholmm}).
Therefore the trivial representation of $\SO(2)$ appears only when $n=0$ and
its multiplicity is 1 for all $\ell$,
and using that $\overline{Y_m^\ell} = (-1)^mY_{-m}^\ell$,
the expansion in \cref{eq:expandh} reduces to
\begin{align*}
  f(g_{\alpha\beta\gamma})
  &= \sum_{\ell \in \N}\sum_{m=-\ell}^{\ell}b_{m}^\ell \rho_\ell^{m0}(g_{\alpha\beta\gamma}) \\
  &= \sum_{\ell \in \N}\sum_{m=-\ell}^{\ell}b_{m}^\ell i^{m}\sqrt{\frac{4\pi}{(2\ell+1)}} \overline{Y_m^\ell(\beta, \alpha)} && (\cref{eq:wig2sph}) \\
  &= \sum_{\ell \in \N}\sum_{m=-\ell}^{\ell}b_{-m}^\ell (-i)^{-m}\sqrt{\frac{4\pi}{(2\ell+1)}} Y_{m}^\ell(\beta, \alpha) \\
  &= \sum_{\ell \in \N}\sum_{m=-\ell}^{\ell} c_m^\ell Y_{m}^\ell(\beta, \alpha).
\end{align*}
We rewrite the expansion as
\begin{align}
  \label{eq:sphharm}
  f(\theta, \phi) &= \sum_{\ell \in \N}\sum_{m=-\ell}^{\ell}\hat{f}_{m}^\ell Y_m^\ell(\theta, \phi),
\end{align}
which shows that the spherical harmonics $Y_m^\ell$ form indeed
an orthonormal basis for $L^2(S^2)$.
The decomposition into the basis is then given by
\begin{align}
  \hat{f}_{m}^\ell &= \int\limits_{x\in S^2} f(x) \overline{Y_m^\ell(x)} \, dx,
\end{align}
where $x \in S^2$ can be parameterized angles $\theta$ and $\phi$.

This concludes our introduction to harmonic analysis.
For more details we recommend \textcite{gallierncharm,dieudonn1980special,folland2016course}.
\textcite{chirikjian2000engineering} present an applied take on the subject.
\section{Applications: equivariant networks}
\label{sec:appl}
We can see a typical deep neural network as a chain of affine operations
$W_i$ whose parameters are optimized, interspersed with nonlinearities $\sigma_i$,
\begin{align}
  f_{\text{out}} = W_n(\cdots \sigma_2(W_2(\sigma_1(W_1f_{\text{in}}))) \cdots).
\end{align}
In convolutional neural networks, these operations are convolutions with an added bias.
The most common nonlinearities are pointwise; one example is the ReLU,
$\sigma(x_i) = \max(x_i, 0)$.

Equivariant neural networks leverage equivariant operations
and symmetries in the data to reduce model and sample complexity.
There are different classes of networks that vary with respect to the group considered,
whether the equivariance is to global transformations or local (patch-wise),
and whether the feature maps are scalar or more general fields.
In this section, we discuss representatives of these classes of networks
in light of the theory presented so far.
We focus on the most interesting cases where the groups $\SO(3)$ and $\SE(3)$
are involved, excluding most of initial work on equivariance to planar transformations.

We are interested in describing how the equivariance is achieved for each case.
It suffices to show it for a single layer ($W_i$ and $\sigma_i$),
since composition of equivariant operations preserves equivariance.
\subsection{Finite group CNNs}
On a finite group,
the counting measure can be used and convolution reduces to summing over each element of the group,
\begin{align}
(f * k)(g) = \frac{1}{|G|} \sum\limits_{x \in G} f(x) k(x^{-1} g).
\end{align}
This simple operation has been successfully applied for rotation equivariance on discrete subgroups of $\SO(3)$;
\textcite{winkels18_g_cnns_pulmon_nodul_detec,worrall2018cubenet} consider the octahedral group of \SI{24} elements, while \textcite{Esteves_2019_ICCV} consider the icosahedral group of \SI{60} elements.
\subsection{The spherical CNNs of \textcite{cohen2018spherical}}
\textcite{cohen2018spherical} introduce a spherical CNN where the inputs are spherical functions $f$
that are lifted to functions on $\SO(3)$ through spherical cross-correlation with a filter $k$,
\begin{align}
\label{eq:sphcorreq}
(f \star k)(g) = \int\limits_{x \in S^2} k(g^{-1}x)f(x) \, dx.
\end{align}
This operation has a pattern matching interpretation.
Suppose $k$ is a rotated version of $f$; then the correlation achieves
its maximum value when $g$ is the rotation that aligns $k$ and $f$.
Note that $f$ and $k$ are functions on $S^2$,
while $f\star k$ is a function on $\SO(3)$.
\begin{proposition}[spherical cross-correlation]
  The spherical cross-correlation between $f,\, k \in L^2(S^2)$ as defined in \cref{eq:sphcorreq}
  can be computed in the spectral domain via outer products of vectors of spherical harmonics coefficients,
  \begin{align*}
    \widehat{(f \star k)}^{\ell} = \overline{\hat{k}^{\ell}} (\hat{f}^\ell)^\top.
  \end{align*}
\end{proposition}
\begin{proof}
We evaluate \cref{eq:sphcorreq} by expanding $f$ and $k$ as in \cref{eq:sphharm},
where $Y^\ell(x) \in \C^{2\ell + 1}$ contains the spherical harmonics of degree $\ell$ evaluated at $x$,
and $\hat{f}^{\ell} \in \C^{2\ell + 1}$ contains the respective coefficients.
We assume real-valued functions (hence the complex conjugation on the first line),
and use the spherical harmonics rotation formula from \cref{eq:sphharmrot}.
\begin{align*}
  (f \star k)(g)
  &= \int\limits_{x \in S^2}
  \sum_{\ell'} (\overline{\hat{k}^{\ell'}})^{\top} \overline{Y^{\ell'}(g^{-1}x)}
    \sum_\ell Y^\ell(x)^{\top} \hat{f}^\ell  \, dx \\
  &= \int\limits_{x \in S^2}
  \sum_{\ell'} (\overline{\hat{k}^{\ell'}})^{\top} \rho_{\ell}(g^{-1}) \overline{Y^{\ell'}(x)}
    \sum_\ell Y^\ell(x)^{\top} \hat{f}^\ell  \, dx \\
  &= \sum_{\ell,\ell'}
    (\overline{\hat{k}^{\ell'}})^{\top} \rho_{\ell}(g)^\top
    \int\limits_{x \in S^2} \overline{Y^{\ell'}(x)}
    Y^\ell(x)^{\top} \hat{f}^\ell  \, dx.
\end{align*}
By orthonormality of the spherical harmonics, $\int_{x \in S^2} \overline{Y^{\ell'}(x)}Y^{\ell}(x)^{\top}$
is the identity $I_{2\ell + 1}$ when $\ell=\ell'$ and zero otherwise. Then,
\begin{align*}
  (f \star k)(g)
  &= \sum_{\ell}
    (\overline{\hat{k}^{\ell}})^{\top} \rho_{\ell}(g)^\top
    \hat{f}^\ell \\
  &= \sum_{\ell} \tr(\hat{f}^\ell(\overline{\hat{k}^{\ell}})^{\top} \rho_{\ell}(g)^\top)
  && (x^\top A y = \tr(yx^\top A)) \\
  &= \sum_{\ell} \tr(\overline{\hat{k}^{\ell}} (\hat{f}^\ell)^\top \rho_{\ell}(g) ).
\end{align*}
where we used the cyclic and transpose properties of the trace in the last part.
The last line is a Fourier expansion of a function on \SO(3) (\cref{eq:fi})
with coefficients given by the outer product of the input coefficients.
This can be restated as
\begin{align}
  \widehat{(f \star k)}^{\ell} = \overline{\hat{k}^{\ell}} (\hat{f}^\ell)^\top,
\end{align}
or in terms of matrix elements,
$\widehat{(f \star k)}_{mn}^{\ell} = \overline{\hat{k}_m^{\ell}} \hat{f}_n^\ell$.
\end{proof}
\begin{remark}
\textcite{makadia2006} show an alternative proof of this result.
The spherical cross-correlation computed this way has further applications
in pose estimation \cite{makadia2006,makadia2007correspondence,esteves-icml19}
and 3D shape retrieval \cite{makadia2010spherical}.
\end{remark}
Returning to \textcite{cohen2018spherical},
only the first layer uses the just described spherical cross-correlation.
In all following layers, features and filters are on $\SO(3)$ and the
$\SO(3)$ cross-correlation is applied,
\begin{align}
(f \star k)(g) = \int\limits_{u \in \SO(3)} f(u)  k(g^{-1} u) \, du,
\end{align}
where the efficient evaluation in the spectral domain is
$(\widehat{f \star k})(\rho) = \hat{k}(\rho)^* \hat{f}(\rho)$.
as shown in \cref{eq:groupcorrspectral}.
The efficient computation for sampled functions
relies on the sampling theorem described by \textcite{kostelec2008ffts}.
\subsection{The spherical CNNs of \textcite{esteves2018learning}}
\textcite{esteves2018learning} introduce a purely spherical convolutional network,
where inputs, filters and feature maps are functions on $S^2$.
The main operation is the spherical convolution
\begin{align}
(f * k)(x) = \int\limits_{g \in \SO(3)} f(g \nu) k(g^{-1} x) \, dg, \label{eq:sphconveq}
\end{align}
where $\nu$ is a fixed point on the sphere (the north pole).
To interpret this operation, we split the integral as in \cref{thm:hmeasure},
which holds since $\SO(3)$ and $\SO(2)$ are unimodular,
\begin{align*}
  (f * k)(x)
  &= \int\limits_{g_{\alpha\beta} \in \SO(3)/\SO(2)} \int\limits_{g_\gamma \in \SO(2)}
    f(g_{\alpha\beta}g_\gamma \nu) k((g_{\alpha\beta}g_\gamma)^{-1} x) \, dg_{\alpha\beta}dg_{\gamma}, \\
  &= \int\limits_{g_{\alpha\beta} \in \SO(3)/\SO(2)}
    f(g_{\alpha\beta} \nu) \left( \int\limits_{g_\gamma \in \SO(2)} k(g_\gamma^{-1}g_{\alpha\beta}^{-1} x) \, dg_{\gamma} \right) dg_{\alpha\beta}. \\
\end{align*}
The inner integral averages $k$ over rotations around the $z$ axis,
resulting in a zonal function (constant on latitudes);
note that this  limits the expressivity of the filters.
The outer integral is then a spherical inner product
where $x$ determines the filter orientation.
\textcite{driscoll1994computing} shows how to compute the convolution
efficiently in the spectral domain. The following lemma will be necessary.
\begin{lemma}
  \label{lemma:dh}
  For $f \in L^2(S^2)$, let $\rho_\ell^{mn}$ be the matrix elements of the unirreps of $\SO(3)$,
  $\nu$ the north pole, and $\hat{f}_n^\ell$ the spherical harmonic coefficient of $f$
  corresponding to $Y_n^\ell$. The following holds
  \begin{align*}
    \int\limits_{u \in \SO(3)} f(u \nu) \overline{\rho_\ell^{mn}(u^{-1})}\, du =
    2\pi\sqrt{\frac{4\pi}{2\ell+1}} \hat{f}_n^\ell
  \end{align*}
  for $m=0$. The integral is 0 otherwise.
\end{lemma}
\begin{proof}
  We apply the change of variables ${u \mapsto ug_{\alpha 0 0}}$ (a rotation around $z$)
  to the following expression,
\begin{align*}
  \int\limits_{u \in \SO(3)} f(u \nu) \overline{\rho_\ell(u^{-1})}\, du
  &= \int\limits_{u \in \SO(3)} f(ug_{\alpha 0 0} \nu) \overline{\rho_\ell(g_{-\alpha 0 0}u^{-1})}\, du \\
  &= \int\limits_{u \in \SO(3)} f(u \nu) \rho_\ell(g_{\alpha 0 0})\overline{\rho_\ell(u^{-1})}\, du,
\end{align*}
where we used that a rotation around $z$ does not move the north pole, $g_{\alpha 0 0}\nu = \nu$.
The left and right hand sides must be equal for all $\alpha$,
and $\rho_\ell^{mm}(g_{\alpha 0 0}) = e^{-i\alpha m}$ (\cref{eq:rholmm}),
so the rows of  $\int_{u \in \SO(3)} f(u \nu) \rho_\ell(u^{-1})\, du$
must be zero for all $m \neq 0$.
Only the matrix values $\overline{\rho_\ell^{0n}(u^{-1})}$ influence the nonzero row,
and $\overline{\rho_\ell^{0n}(u^{-1})} = \rho_\ell^{n0}(u)$ holds.
Using \cref{eq:wig2sph}, we obtain
\begin{align*}
  \int\limits_{u \in \SO(3)} f(u \nu) \overline{\rho_\ell^{0n}(u^{-1})}\, du
  &= \int\limits_{u \in \SO(3)} f(u \nu) \rho_\ell^{n0}(u)\, du \\
  &= \sqrt{\frac{4\pi}{2\ell+1}} \int\limits_{u \in \SO(3)} f(u \nu) \overline{Y_n^\ell(u\nu)}\, du \\
  &= \sqrt{\frac{4\pi}{2\ell+1}} \int\limits_{h \in \SO(2)} \int\limits_{x \in S^2} f(x) \overline{Y_n^\ell(x)}\, dx\, dh \\
  &= 2\pi\sqrt{\frac{4\pi}{2\ell+1}} \hat{f}_n^\ell,
\end{align*}
where \cref{thm:hmeasure} was used in the last passage.
\end{proof}

The spherical convolution is efficiently computed in the spectral domain.
\begin{proposition}[spherical convolution]
  The spherical convolution between $f,\, k \in L^2(S^2)$ as defined in \cref{eq:sphconveq}
  can be computed in the spectral domain via pointwise multiplication of spherical harmonics coefficients,
  \begin{align*}
    \widehat{f * k}_m^\ell = 2\pi\sqrt{\frac{4\pi}{2\ell+1}} \hat{f}_m^\ell\hat{k}_0^\ell.
  \end{align*}
\end{proposition}
\begin{proof}
  Now we replace $k$ in \cref{eq:sphconveq} by its spherical harmonics expansion
  \begin{align*}
    (f \star k)(x)
    &= \int\limits_{g \in \SO(3)} f(g \nu) k(g^{-1} x) \, dg \\
    &= \int\limits_{g \in \SO(3)} f(g \nu) \sum_\ell (\hat{k}^\ell)^\top Y^\ell(g^{-1} x) \, dg \\
    &= \int\limits_{g \in \SO(3)}  f(g \nu) \sum_\ell(\hat{k}^\ell)^\top \overline{\rho_\ell(g^{-1})}Y^\ell(x) \, dg \\
    &= \sum_\ell (\hat{k}^\ell)^\top \left(\int\limits_{g \in \SO(3)}  f(g \nu) \overline{\rho_\ell(g^{-1})} \, dg \right) Y^\ell(x)
  \end{align*}
  Applying \cref{lemma:dh} to the integral within parenthesis, we obtain a
  matrix which has a single nonzero row corresponding to $m=0$,
  so only the $\hat{k}^\ell$ element corresponding to $m=0$ will influence the result.
  We write,
  \begin{align*}
    (f \star k)(x)
    &= \sum_\ell 2\pi\sqrt{\frac{4\pi}{2\ell+1}} \hat{k}_0^\ell (\hat{f}^\ell)^\top Y^\ell(x),
  \end{align*}
  which is the expansion in spherical harmonics of $(f \star k)(x)$.
  The relation
  $\widehat{f * k}_m^\ell = 2\pi\sqrt{\frac{4\pi}{2\ell+1}} \hat{f}_m^\ell\hat{k}_0^\ell$
 follows immediately.
\end{proof}
\begin{remark}
  Observe that only the coefficients $\hat{k}_0^\ell$ appear in the expression,
  which corresponds the coefficients of a zonal spherical function.
  This implies that for any $k$, there is always a zonal function $k'$ such that
  $f * k = f * k'$.
\end{remark}
The efficient computation for sampled functions
on the sphere relies on the sampling theorem as shown by \textcite{driscoll1994computing}.
The spherical convolution as described here is equivalent to the Funk-Hecke formula,
which can be extended to $S^n$; refer to \textcite{gallierdiffgeom} for details.
\subsection{The Clebsch-Gordan networks \cite{kondor2018clebsch}}
\textcite{kondor2018clebsch} generalize the Spherical CNNs by using feature spaces that do not inhabit $S^2$ or $\SO(3)$,
but are just a collection of $\SO(3)$-fragments that transforms as the rows of $\SO(3)$ Fourier components.
We have seen that the Fourier transform of a function $f$ on $\SO(3)$
is a family of matrices $\hat{f}(\rho_\ell)$ as in \cref{eq:ft}.
If $(\lambda_uf)(g) = f(u^{-1}g)$, then $\widehat{(\lambda_uf)}(\rho_{\ell}) = \hat{f}(\rho_\ell) \rho_\ell(x^{-1})$ for all $[\rho] \in \hat{G}$ (a generalization of the shift property of Fourier transforms).
One observation in \textcite{kondor2018clebsch} is that the rows%
\footnote{In \textcite{kondor2018clebsch} these are columns, due to a different definition of the Fourier transform (we follow \textcite{folland2016course})}
of $\hat{f}(\rho_\ell)$ are transformed independently by $\rho_\ell$,
and hence are considered as individual features (called $\SO(3)$-fragments).
Therefore, we can have any number of these fragments for each degree $\ell$
as opposed to only $2\ell+1$ with functions on $\SO(3)$.
Now let $\{f_j^\ell\}$ be a collection of fragments of degree $\ell$,
represented as row vectors.
Each fragment transforms as $f_j^\ell  \mapsto f_j^\ell\rho_\ell(g^{-1})$ upon a rotation $g \in \SO(3)$.
The main operation of the network is the linear combination of fragments,
$f_k^\ell = \sum_j\alpha_j f_j^\ell$,
which is immediately shown to be equivariant ($f_k^\ell \mapsto f_k^\ell \rho_\ell(g^{-1})$).
The weights $\alpha_j$ are the optimized parameters.

Another innovation of \textcite{kondor2018clebsch} is that the nonlinearities are also in the spectral domain, which is potentially more efficient than \textcite{cohen2018spherical,esteves2018learning}.
First we note that the tensor product of irreps is a reducible representation,
which can be represented in terms of irreps through the Clebsch-Gordan transform,
\begin{align}
  \rho_\ell(g) = C_{\ell_1,\ell_2,\ell}^\top (\rho_{\ell_1}(g) \otimes \rho_{\ell_2}(g)) C_{\ell_1,\ell_2,\ell},
  \label{eq:cbt}
\end{align}
where $C_{\ell_1,\ell_2,\ell}$ contains the Clebsch-Gordan coefficients.

The following proposition gives the nonlinear equivariant operation.
\begin{proposition}
  The tensor (Kronecker) product of $\SO(3)$ fragments projected to degree $\ell$
  via the Clebsch-Gordan transform
  \begin{align}
    f_k^\ell = (f_i^{\ell_1} \otimes f_j^{\ell_2})C_{\ell_1,\ell_2,\ell}
    \label{eq:cbnnonlin}
  \end{align}
  is an $\SO(3)$ fragment of order $\ell$ and transform as $f_k^\ell  \mapsto f_k^\ell\rho_\ell(g^{-1})$
  upon rotation.
\end{proposition}
\begin{proof}
  Let us compute the operation when the input fragments are rotated by $g$.
  We use the Kronecker product property $AB \otimes CD = (A\otimes C)(B\otimes D)$
  and \cref{eq:cbt}.
  \begin{align*}
    f_k^\ell = (f_i^{\ell_1} \otimes f_j^{\ell_2})C_{\ell_1,\ell_2,\ell}
    &\mapsto (f_i^{\ell_1}\rho_{\ell_1}(g^{-1}) \otimes  f_j^{\ell_2}\rho_{\ell_2}(g^{-1}))C_{\ell_1,\ell_2,\ell} \\
    &= (f_i^{\ell_1} \otimes f_j^{\ell_2}) (\rho_{\ell_1}(g^{-1})\otimes \rho_{\ell_2}(g^{-1}))C_{\ell_1,\ell_2,\ell} \\
    &= (f_i^{\ell_1} \otimes f_j^{\ell_2})C_{\ell_1,\ell_2,\ell} \rho_{\ell}(g^{-1}) \\
    &= f_k^\ell \rho_{\ell}(g^{-1}).
  \end{align*}
  Therefore, $f_k^\ell$ equivariantly transforms to $f_k^\ell \rho_{\ell}(g^{-1})$ upon rotation by $g$.
\end{proof}
The operation \cref{eq:cbnnonlin} is nonlinear and computed directly on the coefficients,
saving computation of inverse Fourier transforms.
There is flexibility on the choice of pairs of input fragments and output degrees.

\subsection{The $3$D steerable CNNs \cite{weiler20183d}}
\textcite{weiler20183d} present a network that is equivariant to $\SE(3)$ transformations.%
\footnote{Similar ideas are discussed in \textcite{thomas18_tensor_field_networ,kondor18_n_body_networ};
  a 2D predecessor of this work appears in \textcite{CohenW17}.}
Similarly to \textcite{kondor2018clebsch}, it uses a combination of scalar, vector and tensor fields as features, which are transformed by $\SO(3)$ representations.
In contrast with \textcite{cohen2018spherical,esteves2018learning,kondor2018clebsch}, feature fields are on $\R^3$, and cross-correlations on $\R^3$ with constrained filters bring $\SE(3)$ equivariance.
Let $f$ be a feature field that may contain scalars and vectors.
By using the representation of $\SE(3)$ induced by $\SO(3)$ (\cref{def:induced}),
we have $(\pi(rt)(f))(x) = \rho(r)f(r^{-1}(x - t))$
given a transformation $rt \in \SE(3)$ composed of a translation $t$ and rotation $r$,
where $\rho$ is a representation of $\SO(3)$.

The main operation is a cross-correlation over $\R^3$,
which enforces the translation equivariance,
\begin{align*}
  (k \star f)(x) = \int\limits_{\R^3} \overline{k(y - x)}f(y)dy.
\end{align*}
However, since $f$ contains different fields that must transform according to $\SO(3)$ representations,
$k$ is matrix valued and must enforce the equivariance.
By expanding $k \star (\pi_i(rt)f) = \pi_o(rt)(k \star f)$, we arrive at the constraint on $k$,
\begin{align*}
  k(rx) = \rho_o(r) k(x) \rho_i(r)^{-1},
\end{align*}
for all $r\in \SO(3)$, where the indices ``i'' and ``o'' indicate representations acting on input and output feature spaces, respectively.
Let $d_\rho$ be the dimension of the representation space of $\rho$;
we have \fun{k}{\R^3}{\C^{d_{\rho_o} \times d_{\rho_i}}}.

Since every $\SO(3)$ representation can be split in irreps,
we can design the feature fields as a collection of different feature types
that transform by corresponding irreps.
Hence, we can find each block $k_{jk}$ of $k$ independently,
\begin{align}
  k_{jk}(rx) = \rho^{\ell_{o}}(r) k_{jk}(x) \rho^{\ell_{i}}(r)^{-1}, \label{eq:hjk}
\end{align}
where $\rho^{\ell_{i}},\,\rho^{\ell_{o}}$ are the (unitary) irreps of $\SO(3)$ for the input and output degrees, respectively.
If we make $k_{jk}$ separable as in $k_{jk}(x) = R(\norm{x}) \Phi(x/\norm{x})$,
then only $\Phi$, the angular part of $k$, needs to be constrained.
This is because for the radial part, $R(x) = R(rx)$.

We first show a simple expression for $k_{jk}$ that is not the most general,
and then find the general solution.
Let $Y^\ell(x) \in \C^{2\ell + 1}$ denote the spherical harmonics of degree $\ell$ evaluated at $x \in S^2$, and let $\Phi$ be the outer product $\Phi(x) = \overline{Y^{\ell_o}(x)}Y^{\ell_i}(x)^\top$.
Using the spherical harmonics rotation formula $Y^\ell(rx) = \overline{\rho^{\ell}(r)} Y^{\ell}(x)$,
we have
\begin{align*}
  \rho^{\ell_{o}}(r) \Phi(x) \rho^{\ell_{i}}(r)^{-1}
  &= \rho^{\ell_{o}}(r) \overline{Y^{\ell_o}(x)} Y^{\ell_i}(x)^\top \rho^{\ell_{i}}(r)^{-1} \\
  &= [\rho^{\ell_{o}}(r) \overline{Y^{\ell_o}(x)}] [\overline{\rho^{\ell_{i}}(r)}Y^{\ell_i}(x)]^\top \\
  &= \overline{Y^{\ell_o}(rx)} Y^{\ell_i}(rx)^\top \\
  &= \Phi(rx).%
\end{align*}
Therefore, $k_{jk}(x) = R(\norm{x})\overline{Y^{\ell_o}(x/\norm{x})}Y^{\ell_i}(x/\norm{x})^\top$ also satisfies \cref{eq:hjk}.
Since there are no constraints on $R$, it can be parameterized and learned.
However, this is not the most general expression.
A reparametrization allows decomposing the angular constraint and
setting different radial functions to each component, increasing the expressivity of the filter.
To find the general expression for $k_{jk}$, we vectorize it
(make it a vector by concatenating the columns),
and use the relation
\[ AXB = C \implies (B^\top \otimes A)\vectorize(X) = \vectorize(C). \]
Applying this to \cref{eq:hjk}, we get
\begin{align*}
  \vectorize(k_{jk}(rx))
  &= (\rho^{\ell_{i}}(r) \otimes \rho^{\ell_{o}}(r)) \vectorize(k_{jk}(x)).
\end{align*}
Recall that the tensor product of irreps is a representation that decomposes
in irreps according to the Clebsch-Gordan transform (\cref{sec:representations}),
\begin{align*}
  \vectorize(k_{jk}(rx))
  &= C_{\ell_i,\ell_o}^\top
    \left(\bigoplus_{\ell=\abs{\ell_i-\ell_o}}^{\ell_i+\ell_o}\rho_\ell(r)\right)
    C_{\ell_i,\ell_o} \vectorize(k_{jk}(x)),
\end{align*}
where $C_{\ell_i,\ell_o}$ is an orthogonal matrix
that block-diagonalizes the tensor product of irreps.
It is composed by the Clebsch-Gordan coefficients.

Now define $k'_{jk} = C_{\ell_i,\ell_o} \vectorize(k_{jk}(x))$; we have
\begin{align*}
  k'_{jk}(rx)
  &= \left(\bigoplus_{\ell=\abs{\ell_i-\ell_o}}^{\ell_i+\ell_o}\rho_\ell(r)\right)
    k'_{jk}(x),
\end{align*}
where we can see that each $k'_{jk}$ decomposes in $2\min(\ell_i,\ell_o)+1$ parts,
and for each part
\begin{align*}
  k_{jk}^{\prime\ell}(rx)&= \rho_\ell(r) k_{jk}^{\prime\ell}(x).
\end{align*}
Using the spherical harmonics rotation formula from \cref{eq:sphharmrot},
we find that $k_{jk}^{\prime\ell}(x) = \overline{Y^\ell(x/\norm{x})}$ satisfies the constraint.
Again, since the radial component is unconstrained, we can make
$k_{jk}^{\prime\ell}(x) = R_{jk}^\ell(\norm{x})\overline{Y^\ell(x/\norm{x})}$, parametrize and learn each $R_{jk}^\ell$.
This gives $2\min(\ell_i,\ell_o)+1$ radial functions to be learned per block $k_{jk}$,
instead of only one as in our first tentative.
Finally, we obtain the filter blocks $k_{jk}$ by returning to the original basis,
\begin{align*}
  \vectorize(k_{jk}(x)) = C_{\ell_i,\ell_o}^\top \left(\bigoplus_{\ell=\abs{\ell_i-\ell_o}}^{\ell_i+\ell_o}
  R_{jk}^\ell(\norm{x})\overline{Y^\ell(x/\norm{x})}\right),
\end{align*}
and unvectorizing it recovers the $(2\ell_o+1)\times(2\ell_i+1)$ matrix.

Conventional pointwise nonlinearities are only equivariant on scalar fields,
so different ones are required for the vector components of the feature space.
\textcite{weiler20183d} found the best performance with a gated nonlinearity that predicts an extra
scalar field for each feature component, applies a pointwise sigmoid
and multiplies the feature component.
The operation is equivariant since scalar multiplication preserves equivariance.
\section{The general theory of \textcite{kondor18general}}
\label{sec:kt}
We have seen that group convolutions are a way
to learn equivariant representations in neural networks.
\textcite{kondor18general} show that it is the \emph{only} way.
By considering a general neural network with features on homogeneous spaces,
they prove that linear maps between two layers are equivariant if and only if
they have a group convolutional structure.

Following \textcite{kondor18general},
we restrict the derivation to compact discrete groups,
noting that it extends to continuous compact groups
by replacing summations with Haar integrals.
\textcite{kondor18general} define the group convolution between
\fun{f}{G}{\C} and \fun{k}{G}{\C} as%
\footnote{This corresponds to $k * f$ according to \cref{eq:gconv},
  where we follow the convention in \textcite{folland2016course}.}
\begin{align}
(f * k)(g) = \sum_{u \in G} f(gu^{-1})k(u) \label{eq:dgconv}
\end{align}
Since we are interested in functions on homogeneous spaces $G/H$ of a group $G$,
we define the projection $\downarrow$ and lifting $\uparrow$ operators
for \fun{f_1}{G}{\C} and \fun{f_2}{G/H}{\C} as follows,%
\begin{align}
  (\da f_1)(gH) &= \frac{1}{|H|} \sum\limits_{u \in gH} f_1(u),\\
  (\ua f_2)(g) &= f_2(gH),
\end{align}
There is a bijection between a set of cosets and a corresponding homogeneous space as shown in \cref{sec:group}.
We slightly abuse the notation and refer to the homogeneous space as $G/H$,
and its elements as $gH$ for $g\in G$.
We also use $gH$ as a set such that the expression $u \in gH$ make sense.
The map $g \mapsto gH$ is a well defined projection from the group to the homogeneous space.
The map $gH \mapsto g$ consists in an arbitrary choice of coset representative.

We define a more general form of group convolution where the inputs can be on homogeneous spaces,
\fun{f}{G/H_1}{\C} and \fun{k}{G/H_2}{\C},
\begin{align}
  (f * k)(g) = \sum_{u \in G} \ua f(gu^{-1}) \ua k(u).
  \label{eq:ggconv}
\end{align}
If $H_1=H_2=\{e\}$ this is equivalent to \cref{eq:dgconv}.
Since $f$ and $k$ could be on $G$ or $G/H$, there are four possible combinations.

\paragraph{Case I: \fun{f}{G}{\C} and \fun{k}{G/H}{\C}}
In this case, the convolution is
$(f * k)(g) = \sum_{u \in G} f(gu^{-1}) \ua k(u)$.
Let us compute $(f * k)(gh)$ for $h \in H$.
\begin{align*}
  (f * k)(gh) &= \sum_{u \in G} f(ghu^{-1}) k(uH) \\
              &=  \sum_{v \in G} f(gv^{-1}) k(vhH) && (v = uh^{-1}) \\
              &=  \sum_{v \in G} f(gv^{-1}) \ua k(v). && (vhH = vH) \\
              &= (f * k)(g).
\end{align*}
Therefore $f * k$ is constant on cosets $gH \in G/H$
and we can define the convolution as a function on $G/H$,
\begin{align}
  (f * k)(gH) = \sum_{u \in G} f(gu^{-1}) k(uH).
\end{align}
\paragraph{Case II: \fun{f}{G/H}{\C} and \fun{k}{G}{\C}}
In this case, the convolution is
$(f * k)(g) = \sum_{u \in G} \ua f(gu^{-1}) k(u)$.
Consider the space $H\backslash G$; any $u \in G$ can be decomposed as $u = hg$ where
$h \in H$ and $Hg \in H\backslash G$,
\begin{align*}
  (f * k)(g) &= \sum_{u \in G} \ua f(gu^{-1}) k(u) \\
             &=\sum_{Hv \in H\backslash G}\sum_{h \in H} \ua f(g(hv)^{-1}) k(hv) \\
             &= \sum_{Hv \in H\backslash G}\sum_{h \in H} f(g(v^{-1}h^{-1}H)) k(hv) \\
             &= \sum_{Hv \in H\backslash G}f(gv^{-1}H) \sum_{h \in H}  k(hv).
\end{align*}
We can define, without loss of generality, $\fun{\tilde{k}}{H \backslash G}{\C}$,
where $\tilde{k}(Hv) = \sum_{h \in H}  k(hv)$,
and the convolution reduces to
\begin{align}
  (f * k)(g) &= \sum_{Hv \in H\backslash G}f(gv^{-1}H) \tilde{k}(Hv).
\end{align}
This is analogous to the spherical cross-correlation in \cref{eq:sphcorreq},
where $f,\, k$ are on the sphere $S^2 \cong \SO(2)\backslash \SO(3)$ and
$f \star k$ is on $\SO(3)$.

\paragraph{Case III: \fun{f}{G/H_1}{\C} and \fun{k}{G/H_2}{\C}}
In this case, the convolution is
$(f * k)(g) = \sum_{u \in G} \ua f(gu^{-1}) \ua k(u)$.
Using the same procedure of case I we can show that $(f * k)(g) = (f * k)(gh)$
for all $h \in H_2$ so we can treat $f * k$ as a function on $G/H_2$.
Now, using the same procedure of case II,
\begin{align*}
  (f * k)(g) &= \sum_{u \in G} \ua f(gu^{-1}) \ua k(u) \\
                 &= \sum_{H_1v \in H_1\backslash G}f(gv^{-1}H_1) \sum_{h \in H_1}  k(hvH_2). && (u=hv)
\end{align*}
We define $\fun{\tilde{k}}{H_1 \backslash G / H_2}{\C}$
as $\tilde{k}(H_1gH_2) = \sum_{h \in H_1}  k(hgH_2)$,
and the convolution reduces to
\begin{align}
  (f * k)(gH_2) &= \sum_{H_1v \in H_1\backslash G}f(gv^{-1}H_1) \tilde{k}(H_1vH_2).
\end{align}
This is analogous to the spherical convolution of \cref{eq:sphconveq}
where $f$ and $f * k$ are on the sphere $S^2$;
$k$ is also on the sphere, but since it is a zonal function,
it can be seen as on $\SO(2)\backslash \SO(3) / \SO(2)$.

\paragraph{Case IV: \fun{f}{G}{\C} and \fun{k}{G}{\C}}
This case corresponds to \cref{eq:dgconv}
and is analogous to the continuous group convolution of \cref{eq:gconv}.

We now prove that the generalized convolution in \cref{eq:ggconv}
is indeed the most general class equivariant operations.
The following proof is different from the one by \textcite{kondor18general}.
We start from the Fourier transform using matrix coefficients,
impose the equivariance condition
and arrive at the filter constraint.
The main idea of constraining the linear map in the spectral domain is the same.
\begin{proposition}
  A linear map between \fun{f_i}{G/H_i}{\C} and \fun{f_o}{G/H_o}{\C} is
  equivariant to the action of $G$ if and only if it can be written as
  a generalized convolution (\cref{eq:ggconv}) with some filter
  \fun{k}{H_i\backslash G/H_o}{\C}, i.e., $f_o = f_i * k$.
\end{proposition}
\begin{proof}
  The forward direction is proved by a simple change of variables in \cref{eq:ggconv}.
  For the backward direction, consider the linear equivariant map $\varphi$ such that
  $f_o = \varphi(f_i)$.
  Recall that functions on homogeneous spaces can be expanded in Fourier series using the
  same basis as their groups (\cref{sec:fourierh}),
  \begin{align}
    \label{eq:fthspace}
    f_i(gH_i) &= \sum_{[\rho] \in \hat{G}} d_\rho \sum_{i=1}^{d_\rho}\sum_{j=1}^{n_{\rho}^i} \hat{f}_i(\rho)_{ji} \rho_{ij}(g),
  \end{align}
  and since $\varphi$ is linear,
  \begin{align}
    f_o(gH_o) = (\varphi f_i)(gH_o) =
    \sum_{[\rho] \in \hat{G}}d_\rho \sum_{i=1}^{d_\rho}\sum_{j=1}^{n_{\rho}^i} \hat{f}_i(\rho)_{ji}(\varphi\rho_{ij})(g).
    \label{eq:varphiexpansion}
  \end{align}
  Recall that the sum is from 1 to $n_\rho^i$ because we reorder the columns of $\rho$
  so that the $n_\rho^i$ columns that form the basis of $L^2(G/H)$
  appear first (see \cref{sec:fourierh}).
  In this proof, we want to use the same ordering for $G/H_i$ and $G/H_o$ so we
  take the notation $\sum_{j=1}^{n_{\rho}^i}$ to mean
  ``sum over $j$ such that the $j$-th row/column of $\rho$ appear on the basis of $L^2(G/H_i)$''.

  We seek conditions under which $\varphi$ is equivariant.
  The next result will be needed, which follows directly from
  $\rho(g_1g_2) = \rho(g_1)\rho(g_2)$,
  \begin{align*}
  (\lambda_u\rho_{ij})(g)
  &= \rho_{ij}(u^{-1}g) = \sum_{m=1}^{d_\rho} \rho_{im}(u^{-1})\rho_{mj}(g).
  \end{align*}
  Since $\{\rho_{ij}\}$ form a basis of
  the irreducible subspace of $G/H_i$ associated with $[\rho]$,
  we can represent $\varphi$ by its effect on the basis elements,
  \begin{align}
    \label{eq:varphionbasis}
    (\varphi \rho_{ij})(g)
    &= \sum_{k=1}^{d_\rho}\sum_{l=1}^{n_{\rho}^o} \alpha_{ij}^{kl} \rho_{kl}(g),
  \end{align}
  where $n_{\rho}^o$ may be different than $n_{\rho}^i$
  since $G/H_o$ may be different from $G/H_i$.
  Applying the equivariance condition,
  \begin{align*}
    (\lambda_u(\varphi \rho_{ij}))(g)
    &= ((\varphi (\lambda_u\rho_{ij})))(g), \\
    \sum_{k=1}^{d_\rho}\sum_{l=1}^{n_{\rho}^o} \alpha_{ij}^{kl} \rho_{kl}(u^{-1}g)
    &=  \sum_{m=1}^{d_\rho} \rho_{im}(u^{-1})(\varphi\rho_{mj})(g), \\
    \sum_{k=1}^{d_\rho}\sum_{l=1}^{n_{\rho}^o} \alpha_{ij}^{kl}    \sum_{m=1}^{d_\rho} \rho_{km}(u^{-1})\rho_{ml}(g)
    &=  \sum_{m=1}^{d_\rho} \rho_{im}(u^{-1})  \sum_{k=1}^{d_\rho}\sum_{l=1}^{n_{\rho}^o} \alpha_{mj}^{kl} \rho_{kl}(g), \\
    \sum_{k=1}^{d_\rho}\sum_{l=1}^{n_{\rho}^o}\sum_{m=1}^{d_\rho}  \alpha_{ij}^{kl}\rho_{km}(u^{-1})\rho_{ml}(g)
    &=  \sum_{k=1}^{d_\rho}\sum_{l=1}^{n_{\rho}^o}\sum_{m=1}^{d_\rho} \rho_{im}(u^{-1})   \alpha_{mj}^{kl} \rho_{kl}(g).
  \end{align*}
  Since the $\rho_{ij}$ are linearly independent, we fix $l$ on both sides,
  and make $m=n$ on the right and $k=n$ on the left, obtaining
  \begin{align*}
    \sum_{k=1}^{d_\rho}\alpha_{ij}^{kl}\rho_{kn}(u^{-1})
    &= \sum_{m=1}^{d_\rho} \alpha_{mj}^{nl} \rho_{im}(u^{-1}).
  \end{align*}
  Using the linear independence again, we find that
  if $k \neq i$, then $\alpha_{ij}^{kl} = 0$;
  if $m \neq n$, then $\alpha_{mj}^{nl} = 0$;
  if $k = i$ and $m = n$, then $\alpha_{ij}^{il} = \alpha_{nj}^{nl}$.
  Therefore, we only need two parameters to characterize $\alpha$
  (and $\varphi$);
  we define a matrix $A$ such that $A_{lj} = \alpha_{ij}^{il}$, and rewrite \cref{eq:varphionbasis} as
  \begin{align*}
  (\varphi \rho_{ij})(g)
    &= \sum_{l=1}^{n_{\rho}^o} A_{lj} \rho_{il}(g).
  \end{align*}
  Applying this to \cref{eq:varphiexpansion} yields
  \begin{align*}
    f_o(gH_o)
    &= \sum_{[\rho] \in \hat{G}}d_\rho \sum_{i=1}^{d_\rho}\sum_{j=1}^{n_{\rho}^i}
      \hat{f}_i(\rho)_{ji}\sum_{l=1}^{n_{\rho}^o} A_{lj} \rho_{il}(g) \\
    &= \sum_{[\rho] \in \hat{G}}d_\rho \sum_{i=1}^{d_\rho}
      \sum_{l=1}^{n_{\rho}^o} \rho_{il}(g) \sum_{j=1}^{n_{\rho}^i} A_{lj}\hat{f}_i(\rho)_{ji}.
  \end{align*}
  Only  $n_{\rho}^o$ rows and $n_{\rho}^i$ columns of $A$ are used,
  so we can write $A$ as a $d_\rho \times d_\rho$ matrix representing the Fourier
  coefficient of some \fun{k}{H_i\backslash G/H_o}{\C}
  as per discussion in \cref{sec:fourierh}.
  We then make $A = \hat{k}(\rho)$ and write
  \begin{align*}
    f_o(gH_o)
    &= \sum_{[\rho] \in \hat{G}}d_\rho \sum_{i=1}^{d_\rho}
      \sum_{l=1}^{n_{\rho}^o} (\hat{k}(\rho)\hat{f}_i(\rho))_{li} \rho_{il}(g).
  \end{align*}
  The right hand side is exactly the Fourier expansion of
  a function on $G/H_o$ with coefficients $\hat{k}(\rho)\hat{f}_i(\rho)$ (\cref{eq:fthspace}),
  wich are the Fourier coefficients of $f_i * k$ (\cref{thm:conv}).
  Therefore, $f_o = f_i * k$.
  This shows that any equivariant linear map can be written as a generalized convolution.
\end{proof}
\begin{remark}
  This proof works for complex valued functions and filters,
  but is trivially extended to vector-space valued functions,
  as long as each dimension is independently acted upon by the group.
  In other words, the group action does not mix dimensions,
  so the feature vector is a stack of \emph{scalar fields}.
  The case of feature maps as more general fields is discussed in \cref{sec:cohen}.
\end{remark}
\section{Fiber bundles}
\label{sec:fiber-bundles}
Following \textcite{cohen2019general}, we use the language of fiber bundles to present the generalization of the results in \cref{sec:kt}.
We briefly introduce the main concepts in this section;
refer to \textcite{nakahara2003geometry} and \textcite{gallierdiffgeom} for a more complete exposition.

\begin{definition}
A \emph{fiber bundle} $(E, p, B, F, G)$, denoted $\bundle{E}{p}{B}$,
consists of manifolds $E,\,B$ called the \emph{total space} and \emph{base space}, respectively, and a surjective \emph{projection} \fun{p}{E}{B}.
The inverse image $p^{-1}(x)$ is called the \emph{fiber} at $x$, which is isomorphic to a manifold $F$.
The Lie group $G$, called the \emph{structure group}, acts on $F$ from the left.
There exists a set $\{U_i\}$ that is an open cover of $B$ with associated diffeomorphisms $\fun{\phi_i}{p^{-1}(U_i)}{U_i \times F}$, called \emph{local trivializations}.
A \emph{section} \fun{s}{B}{E} is a map satisfying $p(s(x)) = x$ that maps points in the base space to a representative element of the fiber at that point.
\end{definition}
\begin{remark}
  We say that fiber bundles are locally trivial, as $E$ locally looks like the product $B \times F$.
  If $E = B \times F$ everywhere, we call it a \emph{trivial bundle}.
  Trivial bundles have sections defined globally; in general that is not true and
  we work with local sections \fun{s}{B \supseteq U}{E}.
\end{remark}
\begin{example}
  We can see a cylinder as the cartesian product of a circle $S^1$ and line segment $[0, 1]$,
  which is a trivial fiber bundle over $S^1$ with fiber $[0, 1]$.
  A Möbius strip is also a fiber bundle over $S^1$ with fiber $[0, 1]$,
  however it is nontrivial because the fibers are twisted.
  The Möbius strip is only locally homeomorphic to $S^1 \times [0, 1]$.
\end{example}

In a \emph{principal fiber bundle}, we have the additional properties:
(i) the structure group $G$ acts from the right on the total space E,
(ii) the fiber $F$ is homeomorphic to $G$, and
(iii) $E/G$ is diffeomorphic to the base space $B$.
For example, consider a group $E$ and subgroup $G$.
In general, $E/G \times G$ is not isomorphic to $E$,
but we can consider the principal bundle $\bundle{E}{p}{E/G}$,
where $E/G \supseteq U_i  \times G$ locally looks like $E$
(is diffeomorphic to $p^{-1}(U_i)$).

A \emph{vector bundle} is a fiber bundle where fibers are vector spaces.
An example is the tangent bundle $TM$ of an $n$-dimensional manifold $M$,
where the fiber at $p \in M$ is the tangent space $T_pM \cong \R^n$.
\begin{definition}
\label{def:assoc}
An \emph{associated vector bundle} can be associated to a principal bundle \bundle{E}{p}{E/G}.
We consider a vector space $V$ and define an equivalence relation
given by $G$ on $E \times V$ as
$(x, y) \sim_G (xg, \rho(g)^{-1}y)$, where $\rho$ is a representation of $G$ on $V$ and $g \in G$.
The associated vector bundle is then the bundle \bundle{(E \times V)/{\sim_G}}{p_F}{B},
where $p_F([x, y]) = p(x)$.
\end{definition}
\begin{remark}
  Let the columns of a matrix $B$ form a basis for the vector space $\R^n$,
  and let $x$  be a vector represented in this basis.
  Recall how $B$ and $x$ transform under a change of basis $A$:
  $(B, x) \mapsto (BA, A^{-1}x)$.
  The equivalence relation that defines the associated vector bundle follows the same idea:
  $(x, y) \mapsto (xg, \rho(g)^{-1}y)$.
\end{remark}
\begin{definition}[twist]
  \label{def:twist}
  Let \bundle{G}{p}{G/H} be a principal bundle with some section $s$.
  By definition, $H$ acts on $G$ on the right.
  Since the total space is now a group and the base space a homogeneous space,
  we can have actions of $G$ on itself and on $G/H$.
  In general, $gs(x) \neq s(gx)$ for $x \in G/H,\, g \in G $.
  It's easily shown that $gs(x)H=s(gx)H$, i.e., they are on the same fiber.
  We define the \emph{twist} $\fun{\h}{G/H \times G}{H}$ to encode the alignment
  $gs(x) = s(gx)\h(x, g)$.
  This assumes the section $s$ is the same at $x$ and $gx$, but we could also
  define a twist between different sections.
  When considering the coset $eH=H$ where $e$ is the identity of $G$,
  we assume $s(H)=e$ and define ${\h(g) = \h(H, g)}$,
  obtaining the relation $s(gH) = g\h(g)^{-1}$.
\end{definition}
Figure~\ref{fig:twist} illustrates some of the concepts introduced.
\begin{figure}[htbp]
  \centering
  \includegraphics[width=\textwidth]{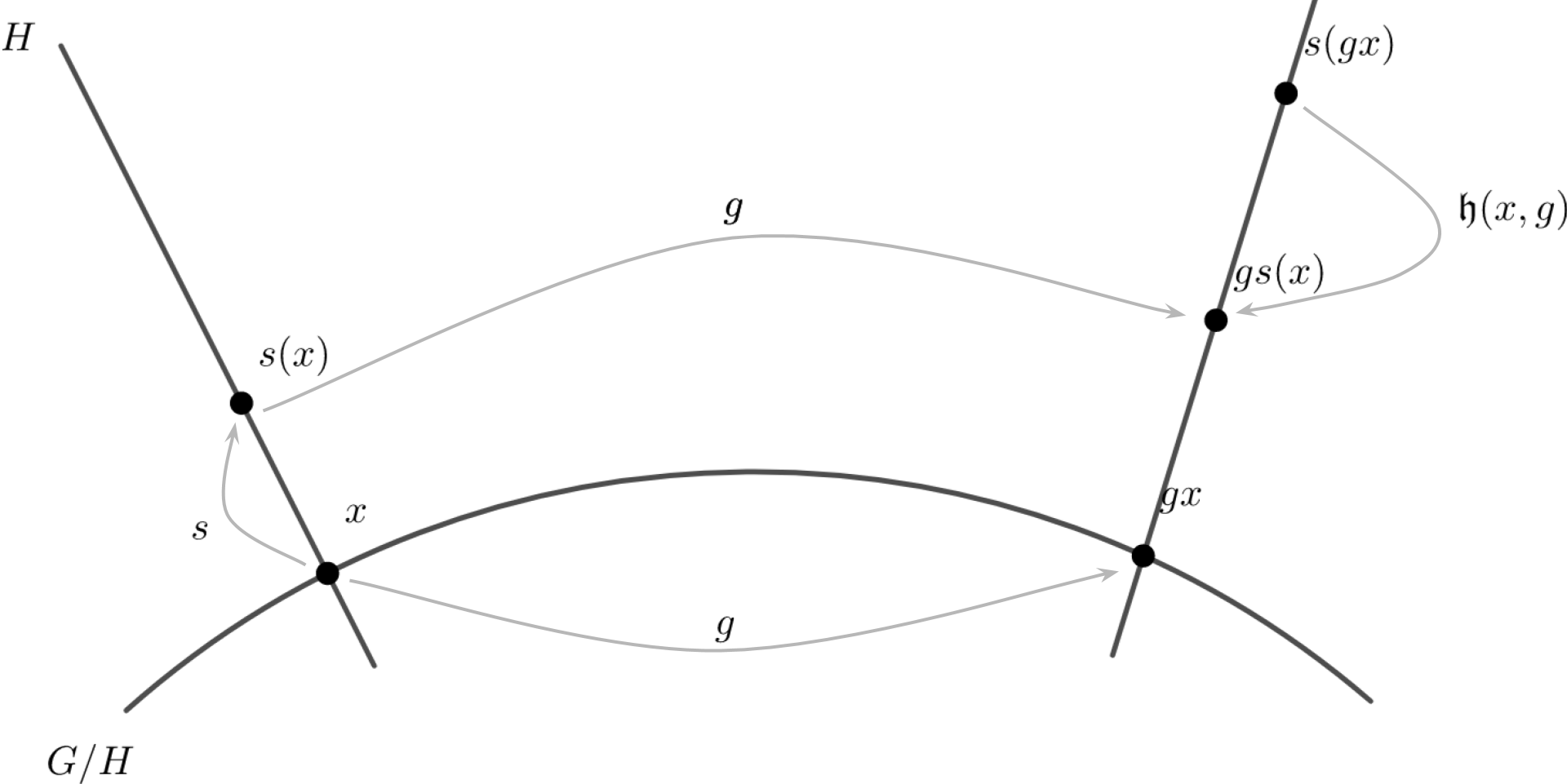}
  \caption{Principal bundle illustration.
    The point $x$ is on the base space $G/H$; an action by $g\in G$  takes it to $gx\in G/H$.
    The section $s$ takes it to $s(x)$ on the total space $G$, on the same fiber associated with $x$.
  Since $gs(x) \neq s(gx)$, we define the ``twist'' $\mathfrak{h}$ such that $gs(x) = s(gx)\h(x, g)$.}
  \label{fig:twist}
\end{figure}
\section{The general theory of \textcite{cohen2019general}}
\label{sec:cohen}
We present a result by \textcite{cohen2019general}, which
(1) generalizes the results in \cref{sec:kt} to the case where features are general fields, and
(2) generalizes \textcite{weiler20183d} from $\SE(3)$-equivariance to a larger class of groups.

It is convenient to describe feature maps of \gcnns\
as sections of associated vector bundles (\cref{def:assoc}).
When the features are vectors on some homogeneous space $G/H$ (i.e., for each $x \in G/H$ we associate a feature vector),
it makes sense to seek equivariance to the group $G$.
For example, if features live on the sphere $S^2 \cong \SO(3)/\SO(2)$
we can consider equivariance to $\SO(3)$~\cite{cohen2018spherical,esteves2018learning}.
The bundle \bundle{G}{p}{G/H} is a principal bundle, where $p(g) = gH$,
so we can construct an associated vector bundle to it
for any vector space $V$ and representation of $H$ on $V$.
There is freedom to choose $V$. For scalar fields we can choose any number of channels
and the trivial representation $\rho = I$ applies.
More generally, a direct sum of vector spaces $V_j$ of arbitrary dimensions can be used
when there are representations $\rho_j$ on each $V_j$.
In this case, we have equivariant feature fields that are sections of the bundle.
\begin{example}
A conventional \cnn\ can be seen as a trivial associated vector bundle $\R^2 \times \R^n$ where
the structural group contains only the identity,
and the trivial representation $\rho=I$ acts on the feature fields with $n$ channels.
We can see the features as a stack of $n$ independent scalar fields.
\end{example}

We have defined the features in a layer of a \gcnn.
The next step is to determine how to transform between layers.
Following most neural networks, we restrict this map to be
a linear operation with learned parameters, followed by a nonlinearity.
Crucially, the transformation must preserve equivariance,
which imposes constraints that we will obtain.

We can see features as functions $\fun{f}{G}{V}$ such that, for all $h \in H$,
\begin{align}
f(gh) = \rho(h^{-1})f(g) \label{eq:mackey}
\end{align}
where $f$ is called a Mackey function.
This works because the relation in $G \times V$
\[(g, f(g)) \mapsto (gh, f(gh)) = (gh, \rho(h^{-1})f(g))\]
holds as prescribed by \cref{def:assoc}.
Note that the space of Mackey functions is exactly the space where the
representations of $G$ induced by $H$ act (see \cref{def:induced}).
From a practical standpoint, defining $f$ this way is redundant and wasteful since
the function is on the total space $G$, but it is useful for algebraic manipulation.

Alternatively, we can consider features as a collection of local functions
\fun{f_j}{G/H \supseteq U_j}{V} on trivializing neighborhoods $U_j$.
We denote $\I_G$ the space of Mackey functions and $\I_U$ the space of local functions.
We can obtain $f_j \in \I_U$ from $f \in \I_G$ using a section of the principal bundle
\begin{align}
  f_j(x) = (\da f)(x) = f(s(x)). \label{eq:g2l}
\end{align}
For the converse, we apply the last relation in \cref{def:twist} to $f(s(gH))$ and use
\cref{eq:mackey}, $f_j(gH) = f(s(gH)) = f(g\h(g)^{-1}) = \rho(\h(g)) f(g)$.
Hence,
\begin{align}
  f(g) = (\ua f_j)(g) = \rho(\h(g))^{-1} f_j(gH). \label{eq:l2g}
\end{align}
\Cref{eq:g2l,eq:l2g} define an isomorphism between $\I_G$ and $\I_U$,
and generalize the $\da{}$ and $\ua{}$ notation of \cref{sec:kt}
(where $\rho$ was always the identity).

The equivariance manifests via the representation of $G$ induced by $H$,
${\pi = \text{Ind}_H^G\rho}$.
For $f \in \I_G$, we write
\begin{align}
(\pi_G(u) f)(g) &= f(u^{-1}g), \label{eq:eqg}
\end{align}
For $f_j \in I_U$, we combine \cref{eq:g2l} and \cref{eq:eqg},
\begin{align}
  (\pi_U(g) f_j)(x) &= (\pi_G(g) f)(s(x)) \nonumber \\
                    &= f(g^{-1}s(x)) \nonumber\\
                    &= f(s(g^{-1}x)\h(x, g^{-1})) \nonumber && (\text{\cref{def:twist}}) \\
                    &= \rho(\h(x, g^{-1})^{-1}) f(s(g^{-1}x)) \nonumber && (\text{\cref{eq:mackey}}) \\
                    &= \rho(\h(x, g^{-1})^{-1}) f_j(g^{-1}x) \label{eq:eql}
\end{align}
\begin{proposition}
  Any equivariant linear map between two spaces of feature fields on homogeneous spaces
  can be written as a cross-correlation operation.
\end{proposition}
\begin{proof}
  We treat the input/output feature spaces as spaces of sections of associated vector bundles to
  the principal bundles \bundle{G}{}{G/H_i} and \bundle{G}{}{G/H_o},
  with vector spaces $V_i$ and $V_o$.
  Consider some linear map \fun{k}{G \times G}{\Hom(V_i, V_o)} applied to
  features $\I_{Gi} \ni \fun{f}{G}{V_i}$,
  \begin{equation}
    (k f) (g) = \int\limits_{u \in G} k(g, u)f(u)\, du. \label{eq:linmap}
  \end{equation}
  Equivariance demands ${\pi_o(u)(k f) = k (\pi_i(u)f)}$.
  Applying this to \cref{eq:linmap}, using \cref{eq:eqg} and changing variables,
  we obtain a constraint on $k$,
  \begin{align*}
    \pi_o(u)(k f)
    &= k (\pi_i(u)f), \\
    \int\limits_{v \in G} k(u^{-1}g, v)f(v)\, dv
    &= \int\limits_{v \in G} k(g, v)f(u^{-1}v)\, dv, \\
    \int\limits_{v \in G} k(u^{-1}g, v)f(v)\, dv
    &= \int\limits_{v \in G} k(g, uv)f(v)\, dv, && (v \mapsto u^{-1}v)\\
    k(u^{-1}g, v)f(v)
    &= k(g, uv), \\
    k(g, v)
    &=k(ug, uv). && (g \mapsto ug)
  \end{align*}
  Now define ${k(u^{-1}v) = k(e, u^{-1}v) = k(u, v)}$.
  Replacing $k(g^{-1}u)$ in \cref{eq:linmap}, we obtain a cross-correlation,
  \begin{equation}
    (k \star f) (g) =  \int\limits_{u \in G} k(g^{-1}u)f(u)\, du. \label{eq:xcorr}
  \end{equation}%
\end{proof}
\noindent This proof assumes $f \in \I_{Gi}$, we still need an expression for $f \in \I_{Ui}$.

Since $k \star f$ must be a Mackey function and satisfy \cref{eq:mackey},
we immediately obtain a left-equivariance condition ${k(hg) = \rho_o(h)k(g)}$
for $g \in G,\, h \in H$.
Since $f$ also satisfies \cref{eq:mackey}, we have
\begin{align*}
  \int\limits_{u \in G} k(g^{-1}u)f(u)\, du
  &= \int\limits_{v \in G} k(g^{-1}vh)f(vh)\, dv && (v = uh^{-1}) \\
  &= \int\limits_{u \in G} k(g^{-1}uh)\rho_i(h^{-1})f(u)\, du,
\end{align*}
which yields a right-equivariance condition $k(gh) = k(g)\rho_i(h)$.
We thus characterize the space of equivariant kernels as
\begin{align}
  K_G = \{\fun{k}{G}{\Hom(V_i, V_o)} \mid&\,k(h_ogh_i) = \rho_o(h_o)k(g)\rho_i(h_i), \nonumber \\
  &\forall g \in G,\, h_i\in H_i,\, h_o\in H_o\}. \label{eq:kg}
\end{align}

Up to this point, we have shown a general expression for equivariant linear maps between
functions in $\I_G$, which are on the group $G$,
and characterized the maps as cross-correlations with kernel functions on $G$.
Representing features and kernels this way is redundant, so we will now find expressions for
 kernels and cross-correlations on the homogeneous space $G/H_i$.

 First, we show that any $k_G \in K_G$ can be represented as a function on $G/H_i$.
We define \fun{k_{H}}{G/H_i}{\Hom(V_i, V_o)} as
\begin{align}
k_H(gH_i) =k_G(s(gH_i)). \label{eq:kappagh}
\end{align}
Applying the last relation in \cref{def:twist} to $k_G(g)$,
\begin{align}
  k_G(g) &= k_G(s(gH_i)\h_i(g)) \nonumber \\
            &= k_G(s(gH_i))\rho_i(\h_i(g))  \nonumber \\
            &= k_H(gH_i)\rho_i(\h_i(g)). \label{eq:kappag}
\end{align}
\Cref{eq:kappagh,eq:kappag} define an isomorphism between spaces of kernels.
Now we show that $k_H$ is left-equivariant, which characterizes the space of kernels on $G/H_i$,
\begin{align*}
  k_H(hx) &= k_G(s(hx)) \\
                     &= k_G(hs(x)\h_i(x, h)^{-1}) && (\text{\cref{def:twist}})\\
                     &= \rho_o(h) k_G(s(x))\rho_i(\h_i(x, h)^{-1}) \\
                     &= \rho_o(h) k_H(x)\rho_i(\h_i(x, h)^{-1}).
\end{align*}
Finally, we find an expression for the cross-correlation with inputs and kernels on the homogeneous space $G/H_i$ and outputs on $G/H_o$.
The following relation will be necessary,
\begin{align}
  \h(g_1g_2) &= s(g_1g_2H)^{-1}g_1g_2 \nonumber \\
            &= (g_1s(g_2H)\h(g_2H, g_1)^{-1})^{-1}g_1g_2 \nonumber \\
            &= \h(g_2H, g_1)s(g_2H)^{-1}g_2 \nonumber \\
            &= \h(g_2H, g_1)\h(g_2). \label{eq:hg1g2}
\end{align}
The strategy is to apply the feature space and kernel isomorphisms in \cref{eq:kappag,eq:l2g,eq:g2l} to the cross-correlation expression in \cref{eq:xcorr},
\begin{align*}
  (k_G \star f) (g) &=  \int\limits_{u \in G} k_G(g^{-1}u)f(u)\, du\\
                       &=  \int\limits_{u \in G} k_H(g^{-1}uH_i)\rho_i(\h_i(g^{-1}u))
                         \rho_i(\h_i(u))^{-1} f_j(uH_i)\, du\\
                       &=  \int\limits_{u \in G} k_H(g^{-1}uH_i)\rho_i(\h_i(uH_i, g^{-1}))f_j(uH_i)\, du. && (\text{\cref{eq:hg1g2}}) \\
\end{align*}
Since $uH_i \in G/H_i$, we can replace the integration limits.
The previous expression still returns a function on $G$, so we apply \cref{eq:g2l} as follows,
\begin{align}
  (k_H \star f_j) (y)  &= (k \star f) (s_o(y)) \nonumber \\
  &= \int\limits_{x \in G/H_i} k_H(s_o(y)^{-1}x)\rho_i(\h_i(x, s_o(y)^{-1}))f_j(x)\, dx,
\end{align}
obtaining the general expression of a $G$-equivariant linear map between feature spaces on homogeneous spaces $G/H_i$ and $G/H_o$, where $G$ is any unimodular locally compact group,
the input features and kernel are defined on $G/H_i$,
and the features take value on any vector space where
there is a representation $\rho_i$ of $H_i$.

One way to further generalize this result is to remove the assumption that features
live on homogeneous spaces of the group so that the group ceases to act transitively on the feature domain.
Features are still vector-valued and group representations still act on them.
This enables the design of equivariant networks for arbitrary graphs and meshes, for example.
The Gauge Equivariant CNNs introduced in \textcite{CohenWKW19} follow this path.

\section*{Acknowledgments}
I am indebted to Jean Gallier who was incredibly helpful through his lectures,
his books (with Jocelyn Quaintance), discussions, and observations about this manuscript.
I'd like to thank Taco Cohen for the thorough and insightful comments and corrections,
Maurice Weiler for corrections, and Shubhendu Trivedi, Edgar Dobriban, Pratik Chaudhari and Kostas Daniilidis for encouraging comments.

\renewcommand{\refname}{\spacedlowsmallcaps{References}} %
\printbibliography

\end{document}